\documentclass{article}

\usepackage[preprint]{jmlr2e}

\usepackage{makecell}

\usepackage{lastpage}
\usepackage{amsfonts, amsmath, amssymb}
\usepackage{mathtools}
\usepackage{bbm}
\usepackage{tikz}
\usepackage{url, xcolor}
\usepackage{natbib}
\usepackage{graphicx}
\usepackage{subfigure}
\usepackage{enumitem}
\usepackage{multirow}
\usepackage{thmtools,thm-restate}
\usepackage{algorithm, algpseudocode}
\usepackage{xurl}

\usepackage[utf8]{inputenc} 
\usepackage[T1]{fontenc}    
\usepackage{hyperref}       
\usepackage{url}            
\usepackage{booktabs}       
\usepackage{nicefrac}       
\usepackage{nameref}
\usepackage{setspace}

\DeclareMathOperator{\Var}{Var}

\newtheorem{deff}{Definition}

\newtheorem{thm}{Theorem}
\newtheorem{prop}{Proposition}

\newtheorem{cor}{Corollary}
\newtheorem{rem}{Remark}

\newenvironment{namedproof}[1]{%
  \par\addvspace{\topsep}\noindent{\bf Proof of #1.} \ignorespaces
}{%
  \hfill\BlackBox\par\addvspace{\topsep}
}

\begin{document}

\title{Revisiting the Neural Tangent Kernel: the role of large width and depth}


\ShortHeadings{Revisiting the Neural Tangent Kernel}{St-Arnaud, Carvalho, Farnadi}

\author{%
  \name William St-Arnaud 
  \email william.st-arnaud@umontreal.ca \\
  \addr Université de Montréal, Mila
  \AND
  \name Margarida Carvalho 
  \email carvalho@iro.umontreal.ca \\
  \addr Université de Montréal, Mila
  \AND
  \name Golnoosh Farnadi
  \email farnadig@mila.quebec \\
  \addr McGill University, Mila
}

\maketitle

\begin{abstract}
    Overparameterized fully-connected neural networks have been shown to behave like kernel models when trained with gradient descent, assuming standard scaling conditions on the width, the learning rate, and the parameter initialization. In the limit of infinitely large widths and infinitesimal learning rate, the obtained kernel provides a description of the learned model's output via a closed-form solution dependent on the architecture and the activation function. The Neural Tangent Kernel, central to this description, remains constant throughout training, a phenomenon that is referred to as ``lazy training'' or within the ``lazy regime''. Prior works show that the ``lazy regime'' leads to non-varying hidden neuron activations in infinitely-wide networks. Moreover, as infinitely-wide networks increase in depth, the Neural Tangent Kernel induces a closed-form solution that is data-independent, hence trivial. The Neural Tangent Kernel seemingly fails to describe the complexity of overparameterized artificial neural networks on two distinct axes: large widths and large depths. In this work, we challenge these two conclusions and open the door to re-evaluating the Neural Tangent Kernel's role in describing the output of overparameterized neural networks. Specifically, we show experimentally that while deviations in the activations of individual hidden neurons vanish, the aggregate norm of these deviations does not. We support this finding with a theoretical result showing that the activations of the last hidden layer do not remain constant with probability 1. Furthermore, we demonstrate that properly scaling the depth and stopping time in infinitely-wide ReLU networks yields a well-behaved, non-trivial output at large dataset sizes. We empirically evaluate the stability of this behavior on large datasets, and we describe the essential properties that enable the generalization of our results to other kernels.
\end{abstract}


\section{Introduction}
Deep neural networks achieve remarkable performance across a diverse span of tasks, from classification to image generation~\citep{lecun1998gradient, silver2017mastering, vaswani2017attention,rombach2022high, silver2016mastering, openai2024o1}. In recent years, a particular observation has been made regarding neural networks that are overparameterized. While classical statistical theory indicated that these large networks would fall prey to overfitting, this conclusion is being challenged empirically~\citep{belkin2021fit}. This ``benign overfitting'' phenomenon is not quite well understood and it has led to works analyzing the learning dynamics of overparameterized models updated with gradient descent~\citep{liu2020linearity, liu2022loss, jacot2018neural}. 

One mechanism which describes the learning dynamics of neural networks involves the linearization of the model through the neural tangent kernel (NTK). This process involves a particular scheme for the initialization of the weight matrices, and it has shown its ability to capture complex dependencies in the infinite-width limit~\citep{jacot2018neural}, although understanding its generalization capabilities is still beyond our reach. Paradoxically, subsequent works have highlighted that under the NTK initialization scheme, this linearized model possesses static features that do not change over time~\citep{Chizatetal2019, yang2021tensor}, calling it a ``lazy regime''. Meanwhile, under other initialization schemes, deep neural networks have demonstrated the ability to learn features and not remain static~\citep{yang2021tensor}. Nevertheless, the fact that the ``lazy regime'' captures highly complex statistical dependencies and shows good generalization capabilities on many tasks~\citep{belkin2021fit} prompts the question: ``how can no features be extracted if generalization is present?'' Elucidating this question can reconcile this apparent paradox by providing new insight into how neural networks form their predictions and generalize to new data. Our first contribution contextualizes the limits of the ``lazy regime'' descriptor, and details an alternative way to recover features in this setting. To achieve this, we establish Proposition~\ref{prop:features_dynamic}, which makes use of the closed-form solution of infinite-width ReLU networks derived in \citet{jacot2018neural}.

In parallel to feature learning in the ``lazy regime'' is the role that depth plays in shaping the output of a model. In order to study this aspect, we can take advantage of the closed-form solutions that are derived in certain architectures. For this reason, we focus on infinitely-wide fully-connected ReLU networks, specifically studying the infinite-depth limit. It is known that the limit NTK involved in the linearization of the model leads to a data-independent predictor~\citep{xiao2020disentangling, seleznova2022analyzing}, which at first glance appears to be a theoretical dead-end. Our second contribution overcomes this challenge, addressing two central aspects of the effect of increasing depths: 1) the convergence of the kernel, established in Proposition~\ref{prop:theta_recursive}, and 2) the output of a deep fully-connected ReLU network under infinitely wide hidden layers and infinitesimal learning rate via Theorem~\ref{thm:deep_stable_ntk}. Our results apply to arbitrary data with support on the sphere. Crucially, we relax several constraints found in prior literature. For instance, our approach avoids assumptions regarding the spectrum of the Hermite expansion, the Mercer decomposition of the kernel, or the non-invertibility condition utilized by~\citet{xiao2020disentangling} (see Theorem~\ref{thm:deep_stable_ntk} and the detailed discussion following it). While the kernel of Proposition~\ref{prop:theta_recursive} does converge to a constant matrix, the output mentioned in 2) can be stabilized by properly scaling the depth with respect to the size of the dataset. We achieve this key result by making use of detailed asymptotic analysis. Finally, our work is distinct from the setting of \citet{hanin2020finite, seleznova2022analyzing}, which relates large stochastic fluctuations in the NTK to a large depth-to-width ratio. We instead study the deterministic limit of the neural tangent kernel when the width is much larger than the depth: while we allow the depth to potentially increase to infinity, the rate at which it does so is much slower than the widths of the hidden layers of a neural network and the size of the dataset. 

Taken together, these contributions provide a more comprehensive theoretical understanding of scaling in neural networks. By detailing how features can be recovered in the infinite-width limit, we nuance the ``lazy regime'' characterization and underline a latent capacity for representation changes over time. Furthermore, by establishing a  well-behaved output for infinite-depth networks despite representation power degeneracy, our work motivates the re-evaluation of the NTK for fully-connected ReLU networks and highlights the importance of using theoretical tools that match experimental settings.

\section{Related work}\label{sec:related_work}


The learning dynamics of overparameterized neural networks have been extensively studied through the lenses of kernel methods and Hessian-based analysis. A key development in this area is the neural tangent kernel, introduced by~\citet{jacot2018neural}, which describes how infinitely wide, fully-connected neural networks trained with gradient descent evolve linearly in function space. While theoretically allowing one to study any loss function, their paper pays specific attention to the Euclidean loss to derive a closed-form solution for the output of the model. The NTK framework formalizes how under common assumptions—particularly Gaussian initialization and wide-layer limits—neural networks behave similarly to kernel methods during training. \citet{arora2019exact} extend the NTK to convolutional neural networks (CNN), further highlighting the framework's capacity to capture learning dynamics. Other architectures and losses, such as cross-entropy, have been explored using the NTK dynamics in \citet{yudivergence, fleissner2025infinite}.

Subsequent work has reinforced the kernel-based interpretation of training dynamics through analyses of the Hessian matrix. \citet{liu2020linearity, liu2022loss} and \citet{belkin2021fit} demonstrate that the loss landscape of overparameterized neural networks often exhibits near-linearity, with low-curvature regions and small-norm Hessians, supporting the NTK-based approximation. These findings suggest that network outputs are relatively stable during training, especially for wide architectures with standard initialization. \citet{lee2020finite} provide an in-depth empirical analysis of NTK models and their performance compared to finite-width neural networks of various architectures (e.g. fully-connected, CNNs). One key observation is that NTKs often outperform finite-width networks, yet are usually surpassed by conventional CNNs. 

The NTK typically requires Monte Carlo estimation of expectations over Gaussian distributions, especially when nonlinearities from activations are involved. This reliance on sampling can be computationally expensive and introduces variance in the resulting kernel evaluations. However, closed-form expressions have been derived for certain activation functions such as ReLU~\citep{arora2019fine} and leaky ReLU~\citep{neuraltangents2020}. These closed-form solutions usually involve Gaussian initialization of the weights, but some works also extend this to other distributions such as the more general set of rotationally invariant distributions~\citep{tsuchida2018invariance}. 

In this work, we adopt the standard setup established by \citet{jacot2018neural} and \citet{arora2019exact}, namely fully-connected networks with ReLU activations trained using the mean squared error (MSE) loss. This specific configuration allows us to leverage existing mathematical machinery, most notably the closed-form expressions for the NTK, as mentioned above. With this in mind, we now contextualize our contributions by reviewing the NTK literature along two primary axes: width and depth. We start in Section~\ref{subsec:rolewidth} with works on the role of width, specifically focusing on the ongoing debate surrounding ``lazy training'' versus feature learning. Then, in Section~\ref{subsec:rolewdepth}, we review works related to the impact of depth on the kernel's stability, spectrum, and predictive phases. Synthesizing these perspectives, Section~\ref{subsec:gaps} identifies the resulting gaps in the current literature and provides a road map for how our contributions address these challenges along both axes.

\subsection{Role of width in the NTK}\label{subsec:rolewidth}
In recent years, the infinite width limit of the NTK has often been characterized as a ``lazy regime'' since individual neurons remain close to their initialization as the width increases. This is often contrasted with the ``feature learning'' or ``rich regime'', where weights are bounded away from their initialization after training. The particular initialization scheme that leads to feature learning is referred to as $\mu P$ parameterization~\citep{yang2021tensor}. For networks with one hidden layer, $\mu P$ has been shown to be equivalent to the mean-field regime~\citep{chizat2018global, meimeanfield}, which evaluates the average output of trained networks under different random initializations (with the same initialization scheme). Under specific conditions, including a one-hidden-layer network, this average output in the mean-field regime converges to the global optimum in the space of distributions over network parameters.

\cite{yang2021tensor} compare  $\mu P$ with the NTK under certain conditions---namely one-hidden-layer network with an identity activation function---and show that the NTK lacks feature learning. However, as pointed out by \citet{jacot2018neural}, even if individual neurons stay arbitrarily close to their initialization, the collective effect of having an increasing number of neurons per layer is highly significant. This is further reinforced by the fact that NTK regression can indeed learn to interpolate complex functions. Therefore, as we will highlight later in this paper, we refine the scope of the assertion that the NTK does not learn any features, while the mean-field regime and $\mu P$ parameterization do.

\subsection{Role of depth in the NTK}\label{subsec:rolewdepth}
While prior work has largely focused on width, less attention has been paid to how depth affects NTK sensitivity to initialization. Among the few works addressing this dimension, \citet{bietti2021deep} show that for the uniform measure on the sphere, the reproducing kernel (c.f. reproducing kernel Hilbert space or RKHS) leads to the same representation power regardless of the network depth. This raises the question of the significance of depth for the NTK. \citet{lee2022neural} present an analysis of deep, narrow multi-layer perceptrons (MLPs) and CNNs as depth goes to infinity. They provide a trainability guarantee, given the right initialization, by showing convergence to a limiting kernel and a vanishing training error. However, this initialization is not close to common initialization methods used in practice, and how depth and generalization capabilities interact in infinitely wide networks remains to be fully understood. 

Further insights are provided by~\citet{nguyen2021tight}, who derive an asymptotic lower bound on the smallest eigenvalue of the NTK through the Hermite expansion of the kernel. This result can be used to derive bounds on the generalization of the model~\citep{arora2019fine} and gives a better grasp on understanding the role that depth plays in both convergence and generalization. \citet{murray2023characterizing} characterize the full spectrum of the NTK via the Hermite expansion for arbitrary datasets on the sphere. They recover an empirical observation that the eigenvalues of the NTK follow power law decay with respect to the size of the training set; \citet{li2024eigenvalue} extend this line of results to general domains beyond the sphere. In addition to the previous result from \citet{jacot2018neural} regarding the convergence to kernel regression in the infinite-width case, they also establish a uniform convergence bound on the projection term of the output of the trained model. This improves upon the pointwise convergence bounds found in \citet{lee2019wide, arora2019exact, allen2019convergence}. As in the initial paper by \citet{jacot2018neural}, the limiting kernel is observed to be deterministic. However, each of these results on the spectrum of the NTK relies on the assumption that data is generated from a probability distribution with a density function, and do not extend to the discrete case. It is therefore of practical interest to relax this assumption and work with an arbitrary dataset. In this regard, \citep{karhadkar2024bounds} provide a lower bound on the smallest eigenvalue of the NTK given a separation condition on the dataset. Their work shows that for arbitrary depth, the smallest eigenvalue is bounded away from zero with high probability. The separation condition can potentially change as the dataset size increases, implying that the smallest eigenvalue decreases towards 0. For this reason, we study the interplay between the rate of decrease of the smallest eigenvalue, the separation condition, and the dataset size.

The deterministic limiting kernel is contrasted with the work of \citet{hanin2020finite, seleznova2022analyzing}, where the authors characterize the NTK directly through the ratio of its second and first moment. They show that if the ratio of depth to width grows large, the NTK has a much higher variance than mean, implying that it is highly stochastic. 

\citet{xiao2020disentangling, poole2016exponential} characterize three different phases describing the behaviour of the limiting kernel of infinitely wide neural networks. They show that the ``ordered phase'' and the ``critical phase'' can lead to generalization when trainable, while the ``chaotic phase'' collapses the mean-predictor to a (data-independent) constant predictor exponentially fast. The NTK described in \citet{jacot2018neural} falls into the category of the ordered phase. In practice, this poses a problem since the limiting predictor (i.e., as depth increases) derived in \citet{xiao2020disentangling} becomes data-independent, and cannot generalize well. As we will see in Section~\ref{sec:limiting_kernel}, it is possible to circumvent the lack of expressivity, and we can recover a well-behaved predictor with the NTK in the ordered phase. Our analysis relies on the assumption that depth and stopping time are scaled properly with increasingly large datasets. Because of this, the work of \citet{seleznova2022analyzing} is outside the scope of this paper, since the ratio of depth-to-width is essentially zero. In practice, small depth-to-width ratios are normally the case.

\subsection{Gaps in the literature and road map}\label{subsec:gaps}

In Section~\ref{sec:notation} and~\ref{sec:background}, we establish notation for the remainder of the article, and we review relevant results in the field. Based on our discussion of the relevant literature, we identify critical gaps regarding the feature learning capabilities and depth limits of the NTK. To situate our contributions within this context, we provide the following road map of our findings along the axes of width and depth: 

\begin{description}
    \item[Re-evaluating the ``Lazy Regime'' (Width Axis)] We re-examine the scope of the ``lazy regime'' characterization for the NTK, contrasting it with the feature learning ability of the mean-field regime and more generally, the $\mu P$ parameterization. While individual neurons remain close to their initialization values~\citep{jacot2018neural, yang2021tensor}, we demonstrate through a toy example that the aggregate effect of wide layers provides significant expressivity. In Section~\ref{subsec:featuresNTK}, we demonstrate experimentally that for finite-width yet highly overparameterized ReLU networks, the last hidden layer exhibits significant movement from its initial state, suggesting that the infinite-width ``lazy'' approximation may not capture the full dynamical behavior of deep architectures. We support our experiments with a theoretical result (Proposition~\ref{prop:features_dynamic}) that shows that the feature correlation of pairs of data points will vary during training. Our findings suggest that the binary distinction between ``lazy'' and ``feature learning'' regimes may be overly simplistic. Moreover, the NTK should not be discarded for its lack of feature learning, especially since its generalization capabilities are well-documented~\citep{lee2020finite, arora2019exact, arora2020harnessing}. As such, it can serve as a good baseline for the performance and behaviour of overparameterized networks. Given closed-form solutions to kernel regression in various architectures, one can use them as proxy for the output of overparameterized models that are otherwise impossible to train in practice.
    \item[Resolving Expressivity and Stability in Deep Networks (Depth Axis)] In Section~\ref{sec:limiting_kernel}, we address the degeneration of the NTK's representation power as network depth increases. While the (normalized) infinite-width NTK approaches an invertible matrix, the predictor loses its ability to distinguish data points that are not in the dataset. This collapse can be prevented by carefully scaling the depth as the size of the dataset increases, and also by selecting an appropriate stopping time (see Theorem~\ref{thm:deep_stable_ntk}). This is contrasted with the usual approach in this case, where $t=+\infty$ or the limit $L \to \infty$ is assumed, and kernel regression becomes unstable or loses any ability to interpolate complex functions, which is not the case for our result. Our proof of Theorem~\ref{thm:deep_stable_ntk} thoroughly analyzes asymptotic bounds in the expression of the model's output, allowing us to establish a set of criteria for stabilizing the solution. Prior works have drawn similar conclusions on the role of depth and early stopping~\citep{cao2021towards, li2024eigenvalue}, yet they required data to be generated from a probability distribution admitting a density function. Our work makes no such assumption and adapts to scenarios where the data is generated from a discrete probability distribution, or where data is selected from a fixed pool (i.e. dataset). We summarize key properties to generalize our results to various other settings. The interplay between dataset size,  depth and early stopping further opens the door to studying the training dynamics of large neural networks in more realistic scenarios.  
\end{description}

Section~\ref{sec:conclusion} contains our concluding remarks on feature learning and the role of depth in the ``lazy regime''. Further experiments and theoretical results are included in the appendix.

\section{Notation}\label{sec:notation}
 To highlight the role of the depth $L$ of a neural network, we denote elements that depend on layer $l \in \{1,\ldots,L\}$ with a superscript $(l)$, e.g., $\kappa^{(l)}$. To emphasize the special role of the last layer in the analysis, the superscript $(L)$ is sometimes used in mathematical objects to remind the reader of its link to a neural network of depth $L$, e.g., $\kappa^{(L)}$ refers to a network of depth $L$, while it can also stand on its own as a mathematical object. The width of a layer $l$ is denoted by $n_l$, with $n_0$ referring to the input dimension. Dependence on a time parameter $t$ and the size $n$ of a dataset is sometimes indicated with a subscript to emphasize their importance, but is otherwise omitted. A dataset of size $n$ is denoted $X$, understood as an $n \times n_0$ matrix, where each row $i$ is written as $x_i^\top$ (i.e., the $x_i$'s are column vectors). For a dataset $X$, it is assumed that all rows are different. The notation $[i]$ for $i \in \mathbb{N}_{>0}$ refers to the set $\{1, \dots, i\}$. We use the notation $f(n) \in O(g(n))$ (respectively, $f(n) \in \Omega(g(n))$) to denote that $g$ is an asymptotic upper bound (respectively, lower bound) for $f$. We also use the notation $\mathcal{O}(g(n)) = O(g(n)) \cap \Omega(g(n))$. Activation functions are denoted by $\sigma$, while uppercase $\Sigma$ is reserved for computing ``covariances'' (see Definition~\ref{def:cov}). The limiting deterministic kernels of \citet{jacot2018neural} are represented using $\Theta_{\infty}^{(L)}$, and $\bar{\Theta}_{\infty}^{(L)}$ for their normalized version (see Definition~\ref{def:bar_theta}); the notation $\kappa$ is used when referring to general kernels and $\bar{\kappa}$ for the normalized version. For the sake of simplicity, we consider neural networks with one-dimensional outputs (i.e., $n_L=1$). Therefore, kernels refer in this context to functions $\mathbb{R}^{n_0} \times \mathbb{R}^{n_0} \to \mathbb{R}_+$. We also write $\kappa\left(A \right)$ for the component-wise application of a kernel $\kappa$ to the entries of a matrix $A$. Specifically, $\kappa\left(XX^\top\right)$ denotes applying the kernel to all pairwise dot products in $X$, where $XX^\top$ is the matrix containing those dot products. Similarly, for any function $g: \mathbb{R} \to \mathbb{R}$, the entry-wise application to a matrix $A$ is denoted by $g(A)$. The notation $A \hookleftarrow_{i,j} A'$ refers to the matrix obtained by replacing column $i$ of matrix $A$ with column $j$ of matrix $A'$. The vector of ones of length $n$ is denoted $\mathbf{1}_n$. Finally, the sphere of dimension $n_0 - 1$ is denoted $S^{n_0 - 1}$.

\section{Background on the NTK and overparameterization}\label{sec:background}

\citet{jacot2018neural} show that, under overparameterization and i.i.d. standard normal weight initialization, a fully-connected neural network of arbitrary depth $L$ exhibits learning dynamics that converge to those of kernel gradient flow in the infinite-width limit. They also provide a recursive formula to compute the kernel $\Theta_{\infty}^{(L)}$ to which the NTK converges (see Theorems 1 and 2 from~\citet{jacot2018neural}). However, as mentioned in the literature review, evaluating $\Theta^{(L)}_{\infty}$ relies on computing high-dimensional expectations and can potentially also be subject to sample inefficiency in the approximation, motivating the search for a more readily computable kernel.

To make the kernel more practical, one may ask whether an efficient closed-form expression can be derived for particular activation functions. Of particular interest is the closed-form representation of $\Theta_{\infty}^{(L)}$ for ReLU activations,  derived by~\citet{arora2019exact}, due to their empirical popularity and methodological appeal in theoretical analysis. In the rest of this article, we study this kernel, i.e., with ReLU activation, $\mathcal{N}(0, 1)$ initialization and infinite-width fully connected networks. To this end, let us first introduce two important definitions and recap the recursive formulation of $\Theta_{\infty}^{(L)}$ by \citet{jacot2018neural}.

\begin{deff}[(mean) Covariance of neurons $\Sigma^{(l)}$]\label{def:cov}
    Let $x$ and $x'$ be two inputs in $\mathbb{R}^{n_0}$.
    The covariances of neurons from inputs $x$ and $x'$ at each layer $l$ are defined recursively as
    \begin{align*}
        \Sigma^{(1)}(x, x') &:= \frac{1}{n_0} x^\top x',  \\
        \Sigma^{(l+1)}(x,x') &:= \mathbb{E}_{z \sim \mathcal{N}(0, \Sigma^{(l)})} [ \sigma(z(x)) \sigma(z(x'))]
    \end{align*}
    where $z \sim \mathcal{N}(0, \Sigma^{(l)})$ is an infinite vector indexed through the notation $z(x)$ and $z(x')$ and  $\mathbb{E}[z(x) z(x')] = \Sigma^{(l)}(x, x')$. We also define the variant of $\Sigma^{(l)}$ where we replace $\sigma$ with its derivative $\dot{\sigma}$:
    \begin{equation*}
        \dot{\Sigma}^{(l+1)}(x,x'):= \mathbb{E}_{f \sim \mathcal{N}(0, \Sigma^{(l)})} [ \dot{\sigma}(f(x)) \dot{\sigma}(f(x')) ].
    \end{equation*}
\end{deff}

\begin{deff}[Neural tangent kernel (NTK)]\label{def:ntk}
    For inputs $x$ and $x'$, the neural tangent kernel of the neural network $f( \cdot; \theta)$ with parameters $\theta \in \mathbb{R}^P$ is given by
    \begin{equation*}
    \Theta^{(L)}(x, x') = \sum_{p=1}^P \frac{\partial f(x; \theta_p)}{\partial \theta_p} \otimes \frac{\partial f(x'; \theta_p)}{\partial \theta_p}.
    \end{equation*}
\end{deff}

Definition~\ref{def:cov} captures the initial covariance of the activations under random initialization. The NTK, on the other hand, governs the network evolution during gradient descent. The following theorem formally connects these concepts by describing the kernel's limiting behavior as layer widths approach infinity.

\begin{thm}[\citet{jacot2018neural}]\label{thm:jacot_ntk}
    Consider a fully-connected neural network of depth $L$ with non-linear activation. In the limit, as layer widths $n_1, \dots, n_{L-1} \to \infty$, the neural tangent kernel $\Theta^{(L)}$ (see Definition~\ref{def:ntk}) converges in probability to a deterministic limiting kernel:
    \begin{equation*}
        \Theta^{(L)} \to \Theta_{\infty}^{(L)} \otimes I_{n_L},
    \end{equation*} where $\Theta_{\infty}^{(l)}$ is defined recursively by
    \begin{align*}
        \Theta_{\infty}^{(1)}(x, x') &:= \Sigma^{(1)}(x, x') \\
        \Theta_{\infty}^{(l+1)}(x, x') &:= \dot{\Sigma}^{(l+1)}(x, x') \Theta_{\infty}^{(l)}(x, x') + \Sigma^{(l+1)}(x, x').
    \end{align*}
\end{thm}

We remark that, although we assume \( n_L = 1 \) for the sake of simplicity, Theorem~\ref{thm:jacot_ntk} is stated in its general form for any output dimension \( n_L \in \mathbb{N} \). This theorem is key in the convergence results obtained in Section~\ref{sec:limiting_kernel} (see proofs of Proposition~\ref{prop:theta_recursive} and Theorem~\ref{thm:convergence_theta_bar} in Section~\ref{sec:limiting_kernel}). 

We also note that we presented the version of Theorem~\ref{thm:jacot_ntk} without biases, i.e., $\beta=0$ in the notation of \citet{jacot2018neural}. In fact, it will later be shown that with non-zero biases ($\beta > 0$), we can still recover our main theoretical contribution regarding depth, as stated in Theorem~\ref{thm:deep_stable_ntk}. We refer the reader to the statement of Theorem~\ref{thm:deep_stable_ntk} and the beginning of Section~\ref{sec:experiments} for more details related to the extension of the results for $\beta>0$. 

We are now ready to state the simplified formula for the covariance and the limiting NTK evaluated on positively correlated inputs.
\begin{prop}\label{prop:correlated_sigma}
    For ReLU activation and perfectly positively correlated inputs $x$ and $x'$, it holds that
    \begin{align*}
        \Sigma^{(L)}(x, x') &= \frac{1}{n_0 2^{L-1}} \lVert x \rVert_2 \lVert x' \rVert_2, \quad
        \dot{\Sigma}^{(L)}(x, x') = \frac{1}{2}
    \end{align*}
    and 
    \begin{align*}
        \Theta_{\infty}^{(L+1)}(x, x') &= \frac{1}{2} \Theta_{\infty}^{(L)}(x, x') + \frac{1}{n_0 2^L} \lVert x \rVert_2 \lVert x' \rVert_2 \\
        &= \frac{L+1}{n_0 2^L} \lVert x \rVert_2 \lVert x' \rVert_2.
    \end{align*}
\end{prop}
\begin{proof}[Proof sketch] 
    Note that $\frac{x^\top x'}{\Vert x \rVert \lVert x' \rVert} = 1$ and the product $\sigma^2(z(x))$, where $z \sim \mathcal{N}(0, 1)$, follows a squared rectified Gaussian distribution. The expectation of the rectified Gaussian is $\frac{1}{2} \Var(z)$, where $\Var(z) =1$, since the mean of $z$ is $0$. Moreover, the zero mean implies $z \geq 0$ with probability $\frac{1}{2}$. We can inductively compute $\Sigma^{(L)}$ and $\dot{\Sigma}^{(L)}$ by plugging these terms back into the recursive formulas (adjusting for the variance $\Sigma^{(l)}(x,x)$ at each layer $l$). 
\end{proof}

To build upon Proposition~\ref{prop:correlated_sigma}, we characterize $\Sigma^{(L)}(x, x')$ and $\Theta_{\infty}^{(L)}(x,x')$ for any correlation between $x$ and $x'$ by introducing a correlation coefficient $\rho^{(L)}$ in Definition~\ref{def:correl_coeff}. These correlations correspond to the correlation between hidden neurons at layer $L$ given inputs $x$ and $x'$.  

\begin{deff}[Correlation coefficient of $\Sigma^{(L)}(x, x')$]
    \label{def:correl_coeff}
    The correlations of neurons from inputs $x$ and $x'$ are defined as
    \begin{equation*}
        \rho^{(L)}(x, x') := \frac{\Sigma^{(L)}(x, x')}{\sqrt{\Sigma^{(L)}(x, x) \Sigma^{(L)}(x', x')}}.
    \end{equation*}
    Note that $\rho^{(L)}(x, x') \in [-1, 1]$.
\end{deff}

In order to compute the kernel $\Theta_{\infty}^{(L)}$ for ReLU activations, one can use Definition~\ref{def:correl_coeff} to compute $\rho^{(L)}$, $\Sigma^{(L)}$, $\dot{\Sigma}^{(L)}$.\footnote{And $\Theta_{\infty}^{(L)}$ using Theorem~\ref{thm:jacot_ntk}.}. In addition, for data points that are outside $S^{n_0-1}$, the homogeneous property of ReLU allows to generalize $\Theta_{\infty}^{(L)}$ to $\mathbb{R}^{n_0}$. These facts are summarized in Proposition~\ref{prop:closedform}.

\begin{prop}[\citet{arora2019exact}\protect\footnotemark]\label{prop:closedform}
    For inputs $x$ and $x'$ with $\rho \in [-1, 1[$, it holds that
    \begin{align*}
        \rho^{(L+1)}(x, x') &= \frac{\sqrt{1 - (\rho^{(L)}(x, x'))^2}}{\pi} + \frac{1}{2} \rho^{(L)}(x,x') \\
        &+ \frac{\rho^{(L)}(x, x') \arcsin{\rho^{(L)}(x, x')}}{\pi} 
        \\
        \dot{\Sigma}^{(L+1)}(x, x') &= \frac{\arcsin{\rho^{(L)}(x, x')}}{2 \pi} + \frac{1}{4}
    \end{align*}
    and 
    \begin{equation*}
        \Theta_{\infty}^{(L)}(x, x') = \lVert x \rVert_2 \lVert x' \rVert_2 \Theta_{\infty}^{(L)}\left(\frac{x}{\lVert x \rVert_2}, \frac{x'}{\lVert x' \rVert_2}\right)
    \end{equation*}
    for a fully-connected neural network with ReLU activation.
\end{prop}
\footnotetext{See also \citet{cho_kernel_methods} for the complete derivation.}

With these results, we achieved our goal of obtaining a closed-form expression for the $ \Theta_{\infty}^{(L)}$ corresponding to an overparametrized (infinite-width), fully-connected ReLU network with no biases. This, in turn, allow us to characterize the output of such neural network, as done in the rest of the section. Indeed, from Proposition~\ref{prop:closedform} and Proposition 2 from~\citet{jacot2018neuralarxiv}, we can immediately observe a few facts regarding the input data:
\begin{enumerate}[label=case~\alph*), align=left]
    \item \label{itm:first} If all datapoints lie on the unit sphere $S^{n_0 - 1}$, the NTK is invertible for $L \geq 2$ (Proposition 2 from~\citet{jacot2018neuralarxiv}).
    \item \label{itm:second} If all datapoints are pairwise not colinear, i.e. $x_i^\top x_j < \lVert x_i \rVert_2 \lVert x_j \rVert_2$ for $i \neq j$, then the NTK is invertible (for $L \geq 2$) since we can project them to different points on the sphere through the canonical projection.
    \item  \label{itm:third} If we map points from $\mathbb{R}^{n_0}$ to the sphere $S^{n_0}$ by embedding them in a space of dimension $n_0 + 1$ and projecting them with the inverse stereographic projection (Definition~\ref{def:inverse_stereographic_projection} in Appendix~\ref{app:otherdefinitions}), the embedding of the datapoints satisfies $\lVert x_i\rVert_2 = 1$ for all $x_i$ in the dataset.
\end{enumerate}

If one of these cases holds, the following proposition provides a closed-form expression for the approximation of the output of a fully-connected neural network.
\begin{prop}[\citet{jacot2018neural}]\label{prop:jacot_f_infty}
    Let $X$ be a dataset of size $n$ (with entries $x_i^\top$) and let $f^*$ and $f_0$, respectively, refer to the learned function and the neural network after the initialization. Further, suppose that the loss is the Euclidean loss. If the limiting kernel $\Theta_{\infty}^{(L)}\left( X X^\top \right)$ is invertible, the output of the neural network converges to
    \begin{equation*}
        f_{\infty}(x) = f_0(x) + \Theta^{(L)}_{\infty}\left( x^\top X^\top \right) \left( \Theta_{\infty}^{(L)} \left( X X^\top \right) \right)^{-1} (y^* - y_0),
    \end{equation*}
    where for $i=1, \dots, n$
    \begin{align*}
        (y^*)_i &=  f^*(x_i), &
        (y_0)_i &= f_0(x_i),
    \end{align*}
    as time $t \to \infty$.
\end{prop}

Motivated by Proposition~\ref{prop:jacot_f_infty}, requiring $\Theta_{\infty}^{(L)} \left( X X^\top \right)$ to be invertible, we identify \textbf{two} regimes of generalization: all data points can lie on either a \textbf{1)} non-compact manifold (i.e. $\mathbb{R}^{n_0})$ or \textbf{2)} a compact manifold (i.e. $S^{n_0-1}$). The compact regime results in a simplifying assumption for the analysis that follows in subsequent sections. We will focus on this regime since without loss of generality (using the homogeneity of ReLU), one can project on $S^{n_0 -1}$ any dataset which contains points in $\mathbb{R}^{n_0}$ that are not pairwise colinear, using the canonical projection. The kernel $\Theta_{\infty}^{(L)} \left( X X^\top \right)$ will thus be invertible. If colinear points exist, an inverse stereographic projection embedding on $S^{n_0}$ will result in an invertible $\Theta_{\infty}^{(L)} \left( X X^\top \right)$.\footnote{In the context of learning the parameters of a neural network, we assume that one first projects onto the sphere and then fixes the projected data during the training phase.}

\section{Feature learning}\label{subsec:featuresNTK}

On the one hand, the NTK is labelled as a ``lazy regime'' since it has been shown that individual weights converge to their initialization values as the width increases~\citep{jacot2018neural,yang2021tensor}. In the infinite-width limit, this is argued to lead to fixed, albeit randomly initialized, weights. As a result, any features extracted---for instance, by taking the neurons at some hidden layer---are inherently random. On the other hand, both the mean-field regime and (more generally) $\mu P$ parameterization have been shown to exhibit feature learning. Therefore, the usefulness or practicality of the NTK has been cast into doubt since it does not appear to show any feature learning ability. In this section, we show that this conclusion requires further discussion.

Concretely, we provide a toy example to demonstrate that vanishing individual neuron variations can still lead to 1) non-vanishing layer variation, and 2) complex modelling at the output neurons. Then, we show that for the ``lazy regime'', we can define features in a meaningful way, where features contain significant statistical correlations between data points, and are subject to evolution throughout the learning process (i.e., not static for each input $x$). Our argument follows experimental results in Table~\ref{tab:feature_deviation} that show that the last hidden layer deviates from its initial value, and we support our observations through Proposition~\ref{prop:features_dynamic}.

\begin{example}
By inspection of Figure~\ref{fig:toy_model}, we can show that defining features as the output of hidden neurons can sometimes fail to capture the underlying dynamics involved in a model's predictions. Note that in the ``lazy regime'' discussed here, the weight matrix acting on layer $l$ is scaled by a factor of $\frac{1}{\sqrt{n_l}}$ for $l \in \{0, 1, \dots , L-1\}$. Note that $n_0=3$ and $n_3=2$, but these values can be arbitrary, as long as they are fixed and only the hidden widths are increased when studying the effect of infinite widths.


\begin{figure}[ht]
    \centering
    \begin{tikzpicture}[
        node/.style={circle, draw, thick, minimum size=0.6cm, inner sep=0pt, fill=white},
        val/.style={draw=none, font=\small, yshift=1pt},
        conn/.style={->, >=stealth, gray!30, line width=0.4pt}
    ]

    \foreach \i in {1,2,3}
        \node[node, label={[val]above:$1$}] (I-\i) at (0, -\i*1.3) {};
    \node[draw=none, font=\small\bfseries] at (0, -5.2) {Input};

    \node[node, label={[val]above:$\frac{1}{n_1^{1+\alpha}}$}] (H1-1) at (2.5, -0.7) {};
    \node[node, label={[val]below:$\frac{1}{n_1^{1+\alpha}}$}] (H1-2) at (2.5, -1.8) {};
    \node[draw=none] (H1-dots) at (2.5, -2.9) {$\vdots$};
    \node[node, label={[val]below:$\frac{1}{n_1^{1+\alpha}}$}] (H1-3) at (2.5, -3.8) {};
    \node[draw=none, font=\small\bfseries] at (2.5, -5.2) {Hidden 1};

    \node[node, label={[val]above:$\frac{1}{n_2^{1+\alpha}}$}] (H2-1) at (5, -0.7) {};
    \node[node, label={[val]below:$\frac{1}{n_2^{1+\alpha}}$}] (H2-2) at (5, -1.8) {};
    \node[draw=none] (H2-dots) at (5, -2.9) {$\vdots$};
    \node[node, label={[val]below:$\frac{1}{n_2^{1+\alpha}}$}] (H2-3) at (5, -3.8) {};
    \node[draw=none, font=\small\bfseries] at (5, -5.2) {Hidden 2};

    \foreach \i in {1,2}
        \node[node] (O-\i) at (7.5, -\i*1.3 - 0.6) {};
    \node[draw=none, font=\small\bfseries] at (7.5, -5.2) {Output};

    \foreach \i in {1,2,3}
        \foreach \j in {1,2,3}
            \draw[conn] (I-\i) -- (H1-\j);

    \foreach \i in {1,2,3}
        \foreach \j in {1,2,3}
            \draw[conn] (H1-\i) -- (H2-\j);

    \foreach \i in {1,2,3}
        \foreach \j in {1,2}
            \draw[conn] (H2-\i) -- (O-\j);

    \end{tikzpicture}
    \caption{Toy model of a deep, arbitrary-width network with identity as the activation function. The values above and below the hidden units represent the initial activations. We suppose that $0 \leq \alpha \ll 1$ and that $n_1 = n_2$.}
    \label{fig:toy_model}
\end{figure}
\end{example}

For simplicity, we assume that $n_1 = n_2$. We further assume that the weight matrices in layer $l$ are scaled as in the ``lazy regime'' with a factor of $\frac{1}{\sqrt{n_{l-1}}}$. In the figure, a wide network with two hidden layers is represented with weights initialized to $\frac{ 1}{ \sqrt{n_{l-1}} n_l^{(2-l)(1 + \alpha)}}$, such that the hidden neurons at layer $l$ have activations $\frac{1}{n_l^{1+ \alpha}}$, where $0 \leq \alpha \ll 1$ and $l \in [2]$. The individual neuron activations converge to their initial values under the lazy dynamics, which themselves converge to $0$. The probability density of the outcome described in Example~\ref{fig:toy_model} does not vanish with increasing widths.\footnote{Note that $\lim_{n_l \to \infty} \left( \exp \left( -\frac{1}{2n_l^{2(1-b_l)(1 + \alpha) + b_l}} \right) \right)^{n_l} > 0$ for $b_l = l - 1$.} In addition, the norm of each hidden layer is equal to $n_l^{-\alpha}$, which converges to $0$ if $\alpha > 0$ and to $1$ if $\alpha = 0$. Nevertheless, the number of connecting paths between the input and the output layer is $n_0 \times n_1 \times n_2 \times n_3$. Therefore, this implies a gradient contribution to an output neuron of the order $n_1^{- 2\alpha}$. This value does not vanish when $\alpha=0$. This behaviour illustrates that, although the gradient contribution per neuron is insignificant, the infinite-width limit can stack these vanishingly small contributions to contribute significantly to the output over time. Note that under a different activation function (such as ReLU), the number of active connections might be smaller than $n_0 \times n_1 \times n_2 \times n_3$.

Example~\ref{fig:toy_model} shows that in theory, the hidden neurons can in some cases fail to capture the complexity of the information useful for a task, but that in aggregate, they lead to complex modelling capabilities. As such, we propose a simple approach to recover features in the ``lazy regime'' for the infinite-width case.

\paragraph{Hidden layer features and the NTK.}

In \citet{yang2021tensor}, the authors use \texttt{word2vec} on the \texttt{text8} dataset to show how the embeddings derived from the infinite-width NTK are useless in another word analogy task (similarity task on embeddings). For this, they show theoretically that for a one-hidden-layer network, the embeddings obtained via the last hidden layer are inevitably random and fixed during training. Their analysis (see Appendix D in \citet{yang2021tensor}) requires the fact that the initial hidden neurons at time $t=0$ are independent for each pair of input $x_1$ and $x_2$. This is the case because tokens in the vocabulary are encoded using a one-hot encoding. Therefore, their dot products are 0 whenever $x_1 \neq x_2$. When compared to the $\mu P$ parameterization, this issue does not arise as the embedding at time $t$ deviates from time $t=0$ at each neuron. They show empirically that the $\mu P$ parameterization will recover a significantly higher than zero accuracy on the word analogy task, while the NTK has an accuracy very close to zero since the features are completely random. 

Because the analysis of \citet{yang2021tensor} is applied to a one hidden layer network, the hidden neurons at time $t=0$ are independent and random, and this is a key reason why the performance on the embedding task is so poor. Since the hidden neurons remain constant throughout training, they remain arbitrary and thus, no statistical patterns or correlations are extracted between similar inputs. However, the same conclusion cannot be drawn directly if more than one hidden layer is used in conjunction with a non-linear activation. In the infinite-width limit, two scenarios could occur: 1) the last hidden layer could deviate in aggregate from its initialization even if its components converge with increasing width, or 2) the embeddings can remain random and constant throughout learning as in the one hidden layer case. 

In order to empirically validate that the ``lazy regime'' can lead to dynamically evolving features, we train a two-layer ReLU network on a subset of the \texttt{text8} dataset. Each token uses a one-hot encoding and the last hidden layer is taken as the features $h_t$ at time $t \in \mathbb{N}$ (i.e., $t$ is the update step counter). Unlike \citet{yang2021tensor}, who obtain static features that do not evolve over time as the width of the hidden layers goes to infinity, we evaluate a two-layer ReLU network to measure feature evolution.  Table~\ref{tab:feature_deviation} presents the deviation of the last hidden layer from its initialization at widths of 10,000 and 30,000. All hidden layer widths are taken to be equal in these scenarios. 

\begin{table}[ht]
    \centering
    \begin{tabular}{c|c|c}
        \toprule
         \textbf{Measure} &  \multicolumn{2}{c}{\textbf{Mean (s.d.)} } \\
          & width = 10K & width = 30K \\
         \hline
         $\lVert h_t - h_0 \rVert_2$ & $4.8989$ $(1.0250)$ & $10.5614$ ($2.3419$)\\
         $\max_i \Vert h_{t,i} - h_{0,i} \rVert_{\infty}$ & $0.0955$ ($0.0189$) & $0.1375$ ($0.0286$)\\
         $\frac{\lVert h_t - h_0 \rVert_2}{\max_i \Vert h_{t,i} - h_{0,i} \rVert_{\infty}}$ & $51.1774$ ($0.7959)$ & $76.5406$ ($1.5052$)\\
         \bottomrule
    \end{tabular}
    \caption{Effect of training on the last hidden layer of a ReLU network. The notation $h_t$ refers to the last hidden layer at update step $t \in \mathbb{N}$. We average over inputs in the dataset. }
    \label{tab:feature_deviation}
\end{table}

From Table~\ref{tab:feature_deviation}, on the one hand, we observe that the individual neuron deviations are close to their initial value (second row values) and this seems to align with the results in \cite{Chizatetal2019, yang2021tensor}. On the other hand, if we instead compare with the total aggregate deviation of the full hidden layer (first row), we can see that it does not converge to $0$. In \citet{jacot2018neural}, the authors crucially remarked that even if individual neurons deviations became small with increasing width, their total contribution is non-negligible to the output of the model. This at first glance seems to be the case for the last hidden layer as demonstrated in Table~\ref{tab:feature_deviation}. Hence, we are motivated to understand these apparently paradoxical observations.

Both \citet{Chizatetal2019} and \citet{yang2021tensor} rely on significant yet subtle assumptions. \citet{Chizatetal2019} show that the activations of hidden neurons in a 2-layer neural network converge to their initial values as the loss function and model output are scaled by an ever smaller constant. This constant does not, however, describe the role that the $\frac{1}{\sqrt{n_{L-1}}}$ factor plays in the output layer of the neural network; in their paper, it is a free parameter that is independent of $\frac{1}{\sqrt{n_{L-1}}}$, while in the setting discussed in this paper, the factor $\frac{1}{\sqrt{n_{L-1}}}$ is tied to the width of the last hidden layer. \citet{yang2021tensor} also provide experimental details that show no useful features are learned by an infinite-width NTK operating in the ``lazy regime''. As mentioned previously, both their theoretical results and their supporting experiments rely on the use of identity activation functions across layers. Identity activations imply that the neural network is a linear function of the input, which cannot capture complex tasks, and enforce that the output is a Gaussian process. In their \texttt{word2vec} task, the one-hot encoding of words $i$ and $j$ implies that $e_i^\top e_j = 0$, when $i \neq j$. In this setting, the output neurons for input $e_i$ will be independent of the output neurons for input $e_j$. This is the key reason that features are completely random and uncorrelated, i.e., do not capture any statistical dependencies. Introducing ReLU activation functions, and having at least 2 hidden layers is sufficient to break the implicit conditions in \citet{yang2021tensor}.\footnote{Compare with Proposition~\ref{prop:closedform} to see that with $x = e_i$ and $x'=e_j$, $i \neq j$, we obtain correlated neurons at hidden layer $2$. This holds even for $x^\top x' = 0$.} The phenomenon described in Table~\ref{tab:feature_deviation} was also observed in \citet{seleznova2022analyzing} for ReLU networks under various hyperparameter initialization schemes falling in the ``lazy regime'', albeit for smaller widths.  

The stated goal of representation learning is to capture statistical dependencies in hidden layers. Next, we detail how this can be done in the ``lazy regime''. \citet{jacot2018neural} describe the evolution of the last hidden layer $h_t$ as

\begin{equation*}
    \partial_t h_{t,i}(x)^\top = \Theta_s^{(L-1)} \left( x^\top X^\top \right) d_{t,i} \left( X \right),
\end{equation*}
where $d_t(x) \in \mathbb{R}^{n_{L-1}}$ is the derivative 
\begin{equation*}
    d_{t}(x)^\top  = \frac{1}{\sqrt{n_{L-1}}} \left( \frac{\partial}{\partial f_t}\mathcal{L}(f_t(x), f^*(x))  \right)^\top W_t,
\end{equation*}
for any loss function $\mathcal{L}: \mathbb{R}^{n_L} \times \mathbb{R}^{n_L} \to \mathbb{R}_{\geq 0}$. If we evaluate the inner product $\langle h_t(x), h_t(x') \rangle$ for $x \neq x'$, we obtain the following:

\begin{prop}\label{prop:features_dynamic}
    Let $\mathcal{L}(f_t(x), f^*(x)) = \lVert f_t(x) - f^*(x) \rVert_2$, i.e., the Euclidean loss. The inner product $\langle h_t(x), h_t(x') \rangle$ is a strictly non-constant function of the training step $t$. Moreover, as $n_{L-1} \to \infty$, the inner product $\langle h_t(x), h_t(x') \rangle$ is not constant with probability 1.
\end{prop}

\begin{proof}
First, let us establish some notation. We denote the standard inner product with $\langle \cdot, \cdot \rangle$. As in Definition~\ref{def:ntk}, the notation $\Theta_s^{(L-1)}$ refers to the NTK up to layer $L-1$ at time $s \geq 0$. For a vector-valued function $g(x) \in \mathbb{R}^m$, the notation $g \left( X \right) \in \mathbb{R}^{n \times m}$ stacks the vectors into a matrix along rows, where each row corresponds to elements of $X$. Subscripts of the form $g_{t,i}$ for a time-dependent vector-valued function $g$, refer the the $i^\text{th}$ component of $g_t$. Remember that $\lim_{n_{L-1} \to +\infty} \Theta_s^{(L-1)} = \Theta_{\infty}^{(L-1)}$ (see Theorem~\ref{thm:jacot_ntk}).  

By expressing $h_t(x)= h_0(x)+ \int_0^t \partial_s h_s(x) ds$, we obtain the following decomposition for $\langle h_t(x), h_t(x') \rangle$:
\begin{align}
    &\phantom{=}
    \langle h_0(x), h_0(x')\rangle + \left\langle h_0(x), \int_0^t \partial_s h_s(x') \, ds \right\rangle + \left\langle \int_0^t \partial_s h_s(x) \, ds , h_0(x') \right\rangle + \left\langle \int_0^t \partial_s h_s(x) \, ds , \int_0^t \partial_s h_s(x') \, ds \right\rangle \nonumber \\
    &= \Sigma^{(L-1)}(x, x') + \sum_{i=1}^{n_{L-1}} \int_0^t h_{0,i}(x) d_{s,i}\left( X \right)^\top \Theta_s^{(L-1)} \left( X x' \right) \, ds \nonumber \\
    &+ \sum_{i=1}^{n_{L-1}} \int_0^t h_{0,i}(x') d_{s,i}\left( X \right)^\top \Theta_s^{(L-1)} \left( X x \right) \, ds \label{eq:expansion of dot product}\\
    &+ \sum_{i=1}^{n_{L-1}} \int_0^t  \Theta_s^{(L-1)} \left( x^\top X^\top \right) d_{s,i} \left( X \right) d_{s,i}\left( X \right)^\top \Theta_s^{(L-1)} \left( X x' \right) \, ds.\nonumber
\end{align}

The matrix $d_{s,i}\left( X \right) d_{s,i} \left( X \right)^\top$ is positive-semidefinite. Therefore, any integrand in the last sum of the expansion of $\langle h_t(x), h_t(x') \rangle$ is zero only when $d_{s,i}\left( X \right)^\top \Theta_s^{(L-1)} \left( X x' \right) = 0$ or $\Theta_s^{(L-1)} \left( x^\top X^\top \right) d_{s,i}\left( X \right) = 0$ for $s \in [0, t]$. If the probability of orthogonality is zero, the integral describes variations in the statistical correlation between the inputs.

When analyzing the infinite-width behaviour, \citet{jacot2018neuralarxiv} show that the matrix $W_t$ of weights at time $t$ converges to its initialization $W_0$, and $h_{t,i}$ converges to $h_{0,i}$ for any $t \geq 0$ and $i \in [n_{L-1}]$ (see the proof of Theorem 2 in Appendix A.2). Specifically, under the Euclidean loss $\mathcal{L}( f_t(x), f^*(x)) = \lVert f_t(x) - f^*(x) \rVert_2$, we obtain in the infinite-width limit that
\begin{align*}
    \sqrt{n_{L-1}} d_{s,i} \left( X \right)^\top \Theta_{\infty}^{(L-1)} \left( X x' \right) &= \left( W_0^\top \right)_{i, \cdot} \left( f_s\left( X \right) - f^* \left( X \right) \right)^\top \Theta_{\infty}^{(L-1)} \left( X x' \right) \\
    \sqrt{n_{L-1}} \Theta_{\infty}^{(L-1)} \left( x^\top X^\top \right) d_{s,i} \left( X \right) &= \Theta_{\infty}^{(L-1)} \left( x^\top X^\top \right) \left( f_s\left( X \right) - f^* \left( X \right) \right) \left( W_0 \right)_{\cdot, i},
\end{align*}
from which we can deduce that with probability 1, these two terms will not be zero\footnote{We use the closed-from solution derived in \citet{jacot2018neuralarxiv} for the Euclidean loss.}. Using a simple union bound argument, both terms will be nonzero with probability 1. In the second term of expression~\eqref{eq:expansion of dot product}, we get
\begin{align*}
     d_{s,i} \left( X \right)^\top \Theta_s^{(L-1)} \left( X x' \right) &= \frac{1}{\sqrt{n_{L-1}}} \left( W_s^\top \right)_{i,\cdot} \left( f_s\left( X \right) - f^*\left( X \right) \right)^\top \Theta_{s}^{(L-1)} \left( X x' \right) \\
    &= \frac{1}{\sqrt{n_{L-1}}} \left( W_s^\top \right)_{i, \cdot} \left( \frac{1}{\sqrt{n_{L-1}}} W_s h_s \left( X \right) - f^* \left( X\right) \right)^\top \Theta_{s}^{(L-1)} \left( X x' \right) \\
    &\to h_{s,i} \left( X \right)^\top \Theta_{\infty}^{(L-1)} \left( X x' \right) \\
    &= h_{0,i} \left( X \right)^\top \Theta_{\infty}^{(L-1)} \left( X x' \right),
\end{align*}
where we used the law of large numbers in the third step and the fact that neuron deviations converge to $0$ as the width increases. Similarly, the integrands in the third term of expression~\eqref{eq:expansion of dot product} converge:
\begin{align*}
     d_{s,i} \left( X \right)^\top \Theta_s^{(L-1)} \left( Xx \right)  \to h_{0,i} \left( X \right)^\top \Theta_{\infty}^{(L-1)} \left( Xx \right)  .
\end{align*} 
Both the second and third term in the expansion of~\eqref{eq:expansion of dot product} will remain constant. We conclude that in the limit $n_{L-1} \to \infty$, the inner product $\langle h_t(x), h_t(x') \rangle$ will vary over time and that this change depends on the terms $\Theta_{\infty}^{(L-1)} \left( x^\top X^\top \right)$ and $\Theta_{\infty}^{(L-1)} \left( X x' \right)$. For a fixed, finite layer width $n_{L-1}$, convergence in probability ensures that the empirical trajectory concentrates tightly around its deterministic, infinite-width counterpart. Consequently, for a sufficiently large choice of $n_{L-1}$, the feature inner product remains strictly non-constant with arbitrarily high probability. This regularizing behavior can be formalized explicitly using standard $\epsilon$-$\delta$ analytical arguments or non-asymptotic concentration inequalities for Gaussian quadratic forms.
\end{proof}

Driven by both the experimental results of Table~\ref{tab:feature_deviation} and Proposition~\ref{prop:features_dynamic}, it is possible to study features learned by a model even in an infinite-width limit. The crucial distinction between our setting and prior work on infinite-width (e.g. \citet{jacot2018neural}) is that we study the layer norm variation $h_t$ when we have a finite dimension $n_{L-1}$ for the last hidden layer, while the widths of the other hidden layers are taken to infinity. As shown previously, $\lVert h_t - h_0 \rVert_2$ will evolve over time, while the terms $\langle h_t(x), h_t(x') \rangle$ will vary at rates that depend on $x, x'$. Ultimately, having a high correlation between $\Theta_{\infty}^{(L-1)} \left( X x \right)$ and $\Theta_{\infty}^{(L-1)} \left( X x' \right)$ drives a variation in statistical dependency relevant to the task , e.g., \texttt{word2vec} with one-hot encoding. 

Ultimately, shifting our perspective to aggregate layer dynamics shows that infinite width does not imply static features. Having established that the NTK framework remains a rich and dynamic proxy for overparameterized models 
along the width axis, we now turn our attention to the orthogonal challenge of scaling: the behavior of these kernels in the deep network limit.

\section{Limiting kernel as depth increases}\label{sec:limiting_kernel}

While we have considered so far a fixed depth $L$ in the infinite-width limit, we have not yet addressed the effect of increasing the depth and the width. Such insights into the additional effect of depth would provide a tangible frontier for the representation power of fully-connected neural networks and their generalization capabilities. As alluded to earlier in Section~\ref{subsec:gaps}, with large depths, $f_t$ becomes data-independent outside of the dataset $X$. This situation fails to demonstrate any generalization capabilities. On the other hand, if we fix $L$, as the dataset size $n$ becomes large, the smallest eigenvalue of the NTK will become arbitrarily close to zero, leading to highly unstable predictions as $t \to + \infty$. In this section, we show that the prediction function $f_t$ from Proposition~\ref{prop:jacot_f_infty}, the projection operator, can be made asymptotically stable and expressive. In other words, when the projection operator is applied to $f_0(X) - f^*(X)$ with a properly chosen stopping time $t$ and depth $L$, the result is not too far from NTK regression with the stable kernel $\frac{1}{4} \mathbf{1}_n \mathbf{1}_n^\top + \frac{3}{4} I_n$, while it also retains the ability to distinguish data points not in the dataset. Note that this setting is different from most works studying the NTK of infinitely-wide neural networks, where $t=\infty$ is used. 

In this section, we reiterate that we always assume ReLU activations for $\Theta^{(L)}_{\infty}$ and $\Sigma^{(L)}$. For a concise list that summarizes the assumptions made in this section, we refer the reader to Appendix~\ref{sec:assumptions}. We start by stating the main results of this section and then, we present their proof. As a warm-up, let us introduce the definition of normalized kernel.

\begin{deff}[Normalization of the $\Theta_{\infty}^{(L)}$ kernel]\label{def:bar_theta}
    For $x, x' \in S^{n_0-1}$, the normalized version of $\Theta_{\infty}^{(L)}$ is defined by
    \begin{equation*}
        \bar{\Theta}_{\infty}^{(L)}(x, x') = \frac{n_0 2^{L-1} \Theta_{\infty}^{(L)}(x, x')}{L}.
    \end{equation*}
\end{deff}

Definition~\ref{def:bar_theta} normalizes the kernel $\Theta_{\infty}^{(L)}$ to  values in $[-1,1]$. The maximum value of $1$ is achieved on the diagonal, i.e., at $\Theta_{\infty}^{(L)}(x,x)$ for $x \in S^{n_0 -1}$. Theorem~\ref{thm:convergence_theta_bar} below states that the limiting NTK of infinitely-wide fully-connected networks converges to the matrix $\frac{1}{4} \mathbf{1}_n \mathbf{1}_n^\top + \frac{3}{4} I_n$ after normalization (see Definition~\ref{def:bar_theta}).

\begin{restatable}[Convergence of $\bar{\Theta}_{\infty}^{(L)}$]{thm}{convthetabar}\label{thm:convergence_theta_bar}
    For any $x, x' \in S^{n_0 -1}$ such that $x \neq x'$, the value $\bar{\Theta}_{\infty}^{(L)}(x, x')$ converges to $\frac{1}{4}$ at a rate $O\left( \frac{1}{L}\right)$ as $L \to \infty$. 
\end{restatable}

 \begin{cor}\label{cor:convergence_theta_bar}
     To observe the $O\left( \frac{1}{L} \right)$ decay rate from Theorem~\ref{thm:convergence_theta_bar}, the depth $L$ should be much larger than $\frac{1}{\sqrt{1- \rho^{(2)}(x^\top x')}}$.
 \end{cor}

In other words, for every pair of inputs $x, x' \in S^{n_0 -1 }$ such that $x \neq x'$, the evaluation $\bar{\Theta}_{\infty}^{(L)} (x^\top x')$ converges to $\frac{1}{4}$, which a priori seems to imply that the predictions of neural regression will converge to the same value for both $x$ and $x'$. For this reason, Theorem~\ref{thm:convergence_theta_bar} can be taken to be a major obstacle to the analysis of $f_t$. However, as Theorem~\ref{thm:deep_stable_ntk} presented below states, for a fixed $x$, we can maintain the expressivity of $f_t$ while increasing the depth $L$, if we properly scale it together with the early stopping $t$ and the dataset size $n$. A finite stopping time ensures that the smallest eigenvalue $\lambda_n$ of $\bar{\Theta}_{\infty}^{(L)} \left( X X^\top \right)$ does not blow up the norm of the predictor. Meanwhile, a proper choice of $L$ ensures that for points not in the dataset, yet in a small neighbourhood of a particular training example, the mapping $f_t$ retains it expressivity. As we collect more data, it is expected that most data points in $S^{n_0 -1}$ will fall close to a data point in $X$. We illustrate the convergence of $\Theta_{\infty}^{(L)}(x,x')$ for $x \neq x'$ in Figure~\ref{fig:stability_off_diagonal}.

 \begin{figure}[htbp]
     \centering
     \begin{minipage}[b]{0.49\textwidth}
         \includegraphics[width=\textwidth]{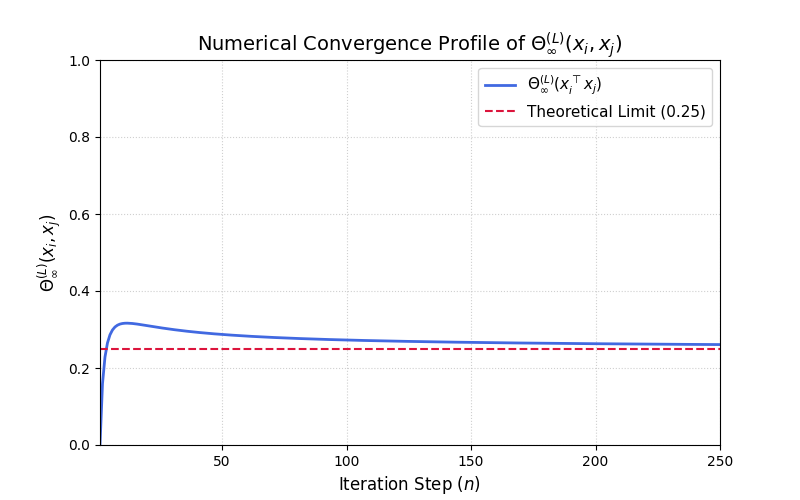}
         \centering
         a) $\rho^{(1)}(x^\top x') = 0$
     \end{minipage}
     \hfill
     \begin{minipage}[b]{0.49\textwidth}
         \includegraphics[width=\textwidth]{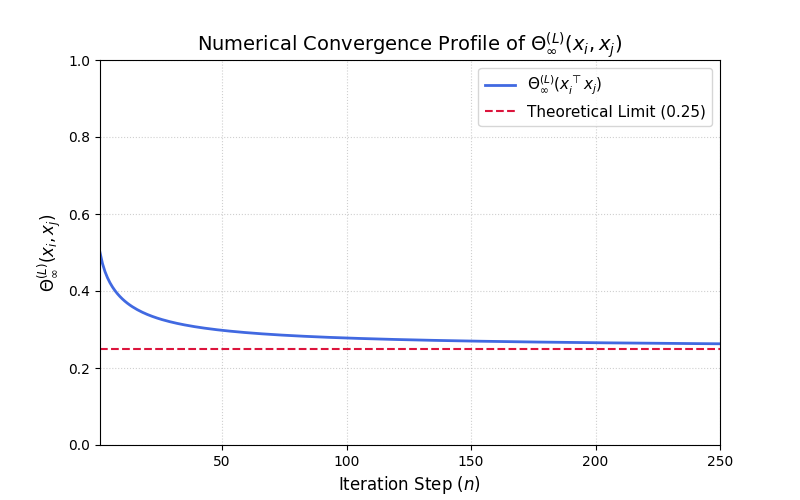}
         \centering
         b) $\rho^{(1)}(x^\top x') = 0.5$
     \end{minipage} \\
     \begin{minipage}[b]{0.49\textwidth}
         \includegraphics[width=\textwidth]{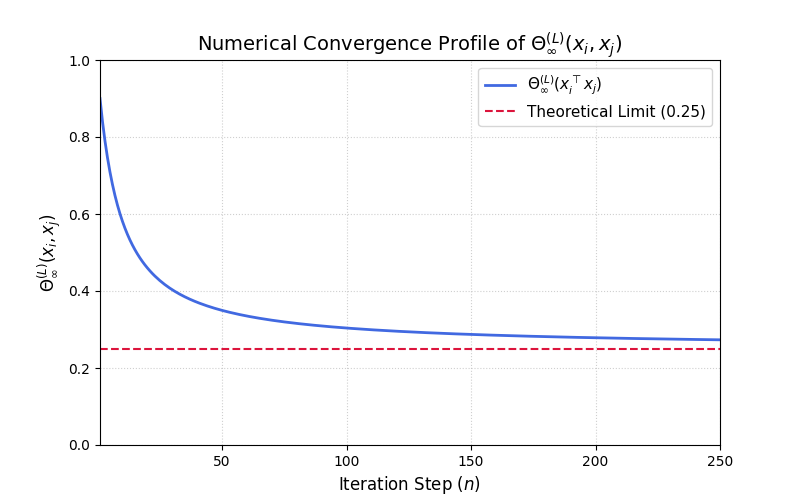}
         \centering
         c) $\rho^{(1)}(x^\top x') = 0.9$
     \end{minipage}
     \hfill
     \begin{minipage}[b]{0.49\textwidth}
         \includegraphics[width=\textwidth]{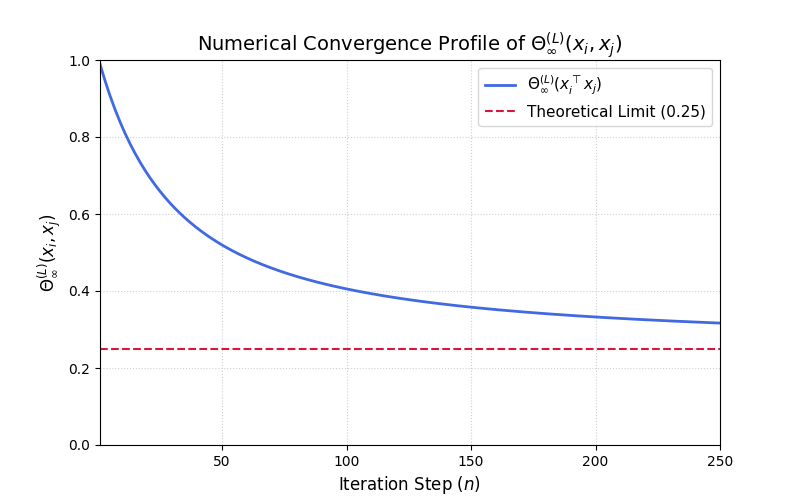}
         \centering
         d) $\rho^{(1)}(x^\top x') = 0.99$
     \end{minipage}
     \caption{Convergence of $\Theta_{\infty}^{(L)}(x, x')$ (blue line) to $\frac{1}{4}$ (red line) with varying initial values of $\rho^{(1)}(x^\top x')$.}
     \label{fig:stability_off_diagonal}
 \end{figure}

\begin{thm}\label{thm:deep_stable_ntk}
In an infinitely-wide fully-connected ReLU network of depth $L$, for any evaluation point $x \in S^{n_0-1}$ satisfying $x \neq \pm x_i$ for all $i \in [n]$, the vector $\frac{1}{4} \mathbf{1}_n - \bar{\Theta}_{\infty}^{(L)} \left( X x \right) \in \mathcal{O}\left(\frac{1}{L} \right)$ component-wise. Moreover, as we increase $n$ and keep $n_0$ fixed, two possible scenarios arise:
\begin{enumerate}[label=\alph*)]
    \item There exists $0 < \delta < 1$ such that for any $n$, $x_i^\top x_j \leq 1 - \delta$, where $i,j \in [n]$ and $i \neq j$,
    \item $\lim_{n \to \infty} \max_{i,j \in [n] \mid i \neq j} x_i^\top x_j = 1$.
\end{enumerate}

In case a), $\lVert P^{(L)}_{S}  v - P_{S}v \rVert_2 \in O \left( \frac{1}{L} \right)$ for any vector $v$ and $S \in \{\{1\}, \{2, \dots, n\}\}$, where $P_S^{(L)}$ (resp. $P_S$) projects onto the span of the eigenvectors of $\bar{\Theta}_\infty^{(L)} \left( X X^\top \right)$ (resp. $\frac{1}{4} \mathbf{1}_n \mathbf{1}_n^\top + \frac{3}{4} I_n$ indexed by $S$. For any $t$, $\lVert f_t(x) - \hat{f}_t(x) \rVert_2 \in O(n)$ for $x \in S^{n_0-1}$, where $\hat{f}_t$ is defined analogously to $f_t$, with its NTK given by $\frac{1}{4} \mathbf{1}_n \mathbf{1}_n^\top + \frac{3}{4} I_n$.

In case b), performing early stopping with $t$ will stabilize the projection on eigenvectors of $\bar{\Theta}_{\infty}^{(L)} \left(X X^\top \right)$. Specifically, by setting $t \in o \left( \frac{1}{\lambda_n} \right)$, where $\lambda_n$ is the smallest eigenvalue of $\bar{\Theta}_\infty^{(L)} \left( X X^\top \right)$, the norm of the projection operator (see Section~\ref{sec:auxiliary_results}) remains in $O \left( n \right)$.

In addition to the bounds in a) and b), for each $x \in S^{n_0 -1}$, the projection operator induced by $x$ is non-trivial whenever $L \in O\left( \frac{1}{\sqrt{1- \rho^{(2)} \left( x_i^\top x_j \right)}} \right)$ for some $x_i, x_j \in X$ with $x_i \neq x_j$. 
\end{thm}

Theorem~\ref{thm:deep_stable_ntk} demonstrates that in deep ReLU networks, predictions can still be stable (neither vanishing nor exploding) provided the experimenter calibrates early stopping and depth.  Practitioners can tune these hyperparameters depending on the level of representation variance desired for a specific task. In practice, the depth $L$ is much smaller than the dataset size $n$. The experimenter has the choice to use the hyperparameters $n, t, L$ to ensure stability of the model's predictions in the ``lazy regime.''~\footnote{Varying the hyperparameter $n$ depends one the experimenter's ability to obtain new data points.} By the Davis-Kahan theorem, the subspace spanned by the $n-1$ eigenvectors associated with the smallest eigenvalues of $\bar{\Theta}_{\infty}^{(L)} \left( X X^\top \right)$ concentrates tightly around the subspace spanned by the $n-1$ eigenvectors of $\mathbf{1}_n \mathbf{1}_n^\top$ with smallest eigenvalue (i.e. $\frac{3}{4}$). In fact, the theorem provides a bound on the sine of the angle $\theta$ between the two subspaces: $\sin(\theta) \leq \frac{1}{n} \lVert \mathbf{1}_n \mathbf{1}_n^\top - \bar{\Theta}_{\infty}^{(L)} \left( X X^\top \right) \rVert_F$. Similarly, the eigenvectors with largest eigenvalues can be shown to be very close. It is therefore possible to approximate $f_t(x)$ for any $x$ uncorrelated to $X$ (i.e. $\lVert Xx\rVert_{\infty} \approx 0$) by applying the eigendcomposition of $\mathbf{1}_n \mathbf{1}_n^\top$ directly; this idea uses the fact that for large $L$, the largest eigenvalue is mapped to $0$, while the other eigenvalues are approximately mapped to $-t$ by the exponential map. Therefore $f_t$'s behaviour interpolates between two regimes: the first is described by the trivial predictor $\hat{f}_t$ for data uncorrelated to $X$, and the second is described by the non-trivial part of $f_t$ on data correlated with $X$. As we collect more data points and $n$ increases, individual points on $S^{n_0 -1}$ become more correlated with $X$ and the first regime becomes less likely, in which case it is possible to safely increase the depth $L$ via the bound in Theorem~\ref{thm:deep_stable_ntk}.

To fully contextualize the implications of this stability result, we next discuss (i) our circumvention of data-independence relative to prior literature, (ii) our extension to non-compact domains, and (iii) the applications of our analysis to mean-field particle limits and alternative architectures.

First, while $\bar{\Theta}_{\infty}^{(L)} \left( X X^\top \right)$ approaches an invertible matrix, it becomes trivial for $x \notin X$ when $L \to \infty$. Nevertheless, our proof shows that $f_t$ can remain expressive. This distinguishes our setting from that of \citet{xiao2020disentangling}, where $t \to \infty$, and $f_t$ is data-independent. \citet{li2024eigenvalue} provide a similar result describing the importance of early stopping, yet their work does not specifically analyze the role of depth. Furthermore, their work requires the assumption that the dataset is generated from a probability distribution admitting a density function, which is not the case for Theorem~\ref{thm:deep_stable_ntk}.

Second, for ReLU networks, it is possible to easily extend this stability guarantee to the non-compact regime (i.e., general domain $\mathbb{R}^{n_0}$). By Proposition~\ref{prop:closedform}, there is a closed-form to $\Theta_{\infty}^{(L)}$ for general data points in $\mathbb{R}^{n_0}$. In the statement of the proposition, the canonical projection on the sphere is provided, but a similar result is obtained for a stereographic projection.

Finally, while works such \citet{bietti2021deep} and \citet{li2024eigenvalue} show that the representation power of $\Theta^{(L)}_{\infty}$ does not change as $L \to \infty$, applying our scaling result to the mean-field regime of \citet{chizat2018global}, we find that each particle approaches the deterministic limit by inspection of Proposition~\ref{prop:jacot_f_infty}. It is therefore possible to analyze the many-particle limit of very wide and deep fully-connected neural networks since these are well approximated by $f_\tau$ for a proper stopping time $\tau$~\citep{li2024eigenvalue}.\footnote{The particles are given by different initializations $f_0(x), f(X)$.} In addition, the analysis of Theorem~\ref{thm:deep_stable_ntk} can be adapted to other kernels that arise from other architectures such as CNNs.

In order to prove both theorems, we provide technical results and definitions in what follows. Lemma~\ref{lem:convergence_rho}, Definition~\ref{def:hfun} and Proposition~\ref{prop:theta_recursive} form the building blocks of Theorem~\ref{thm:convergence_theta_bar}, while Proposition~\ref{prop:matrix_exp} is essential to proving Theorem~\ref{thm:deep_stable_ntk}. Proposition~\ref{prop:unfiform_close} and Example~\ref{ex:deep_stable} provide a specific application to a dataset uniformly sampled on $S^{n_0 -1}$ and shows that with high probability, data points will satisfy a separation criterion which allows us to apply Theorem~\ref{thm:deep_stable_ntk} part a) to choose an appropriate depth $L$.

\subsection{Auxiliary results and definitions}\label{sec:auxiliary_results}
The proofs of the results provided here can be found in Appendix~\ref{sec:additional_lemmas}. The following lemma establishes that $\rho$ converges to 1 for each pair of data points as $L$ goes to infinity. This result is a key ingredient in the propositions and theorems that follow.

\begin{restatable}[Convergence of $\rho^{(L)}$]{lem}{lemmaappendix}\label{lem:convergence_rho}
    If $\rho^{(1)}(x, x') \in \left]-1, 1\right[$, then $\rho^{(L)}(x, x') \to 1$ as $L \to \infty$.
\end{restatable}

Below, we define a function that  will help us to rewrite the expression $\bar{\Theta}_{\infty}^{(L)} \left( X X^\top \right)$ obtained in Theorem~\ref{thm:convergence_theta_bar} in a convenient form.

\begin{deff}\label{def:hfun}
We define the function $h: [-1, 1] \to \mathbb{R}$ as
\begin{equation*}
        h(z) = \frac{z \arcsin(z)}{\pi} + \frac{\sqrt{1 - z^2}}{\pi} + \frac{z}{2}.
    \end{equation*}
\end{deff}

\begin{figure}[ht]
    \centering
    \includegraphics[width=0.85\linewidth]{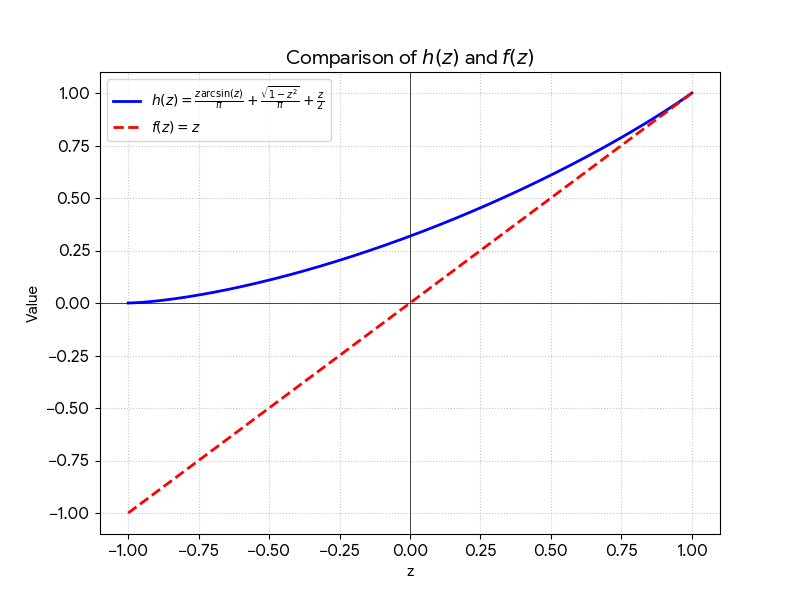}
    \caption{Plot of the function $h$ (blue) as a function of $z$. For reference, we plot the identity function in red. Note that by composition, $\lim_{L \to \infty} h^L(z) \to 1$ for any $z \in [-1, 1]$. This key property is used in the proof of Theorem~\ref{thm:convergence_theta_bar} (see Section\ref{sec:proof_convergence_theta_bar})}
    \label{fig:hfun}
\end{figure}

Figure~\ref{fig:hfun} illustrates the behaviour of function $h$ from Definition~\ref{def:hfun}. From the figure, it can be seen that $h$ is strictly increasing over $[-1,1]$, and that it converges to $1$ under composition for any $z \in [-1,1]$. The function $h$ is a key ingredient in the convergence of Theorem~\ref{thm:convergence_theta_bar}. Exploiting this functional structure yields an explicit recurrence relation for the normalized deep kernel.

\begin{restatable}[Alternative formulation of $\bar{\Theta}_{\infty}^{(L)}$]{prop}{thetarecursive}\label{prop:theta_recursive}
    The equality
    \begin{align*}
        \bar{\Theta}_{\infty}^{(L+1)}(x, x') &= \frac{L}{L+1} h'\left(\rho^{(L)}(x, x')\right) \bar{\Theta}_{\infty}^{(L)}(x, x') \\
        &+ \frac{1}{L+1} h\left(\rho^{(L)}(x, x') \right)
    \end{align*}
    holds $\forall x, x' \in S^{n_0 -1}$. Moreover, the values in the normalized kernel are all found in the interval $\left[ 0, 1 \right]$.
\end{restatable}

We explicitly provide the early stopping version of the output $f_t$ of a fully-connected ReLU network trained with Euclidean loss. By Theorem~\ref{thm:convergence_theta_bar}, for large $L$ the space spanned by the $n-1$ smallest eigenvalues will get close to the space spanned by the $n-1$ eigenvectors associated with the smallest eigenvalues of the matrix $\mathbf{1}_n \mathbf{1}^\top$. These two subspaces can be rotated into the other with an $n \times n$ rotation matrix $R$ that is close to the identity. Writing the eigenvector decomposition of $\bar{\Theta}_{\infty}^{(L)}\left( X X^\top \right)$ as $V^{(L)} \Lambda^{(L)} V^{(L)\top}$, we establish below the closed-form of the prediction function at time $t$.

\begin{prop}\label{prop:matrix_exp}
    For a square matrix $A$, let $\exp(A)$ denote its matrix exponential, i.e., 
    \begin{equation*}
        \exp(A) = \sum_{k=0}^\infty \frac{A^k}{k!}.
    \end{equation*}
    Given kernel $\bar{\Theta}_{\infty}^{(L)}$ and dataset $X$, the function $f_t$ is expressed by 
    \begin{align*}
        f_t(x) = f_0(x) + \bar{\Theta}_{\infty}^{(L)} \left( x^\top X^\top \right)
        \left( \bar{\Theta}_{\infty}^{(L)}  \left( X X^\top\right) \right)^{-1} \left( \exp\left( -t \bar{\Theta}_{\infty}^{(L)} \left( X X^\top \right) \right) - I_n \right) \left( f_0(X) - f^*(X)\right).
    \end{align*}
    If $\{\lambda_k\}_{k=1}^n$ are the eigenvalues of $\bar{\Theta}^{(L)} \left( X X^\top \right)$, then $\frac{e^{-t\lambda_k} - 1}{\lambda_k}$ are the eigenvalues of the matrix
    \begin{equation*}
        \left( \bar{\Theta}_{\infty}^{(L)} \left( X X^\top \right) \right)^{-1} 
        \left( \exp \left( -t \bar{\Theta}_{\infty}^{(L)} \left( X X^\top \right) \right) - I_n \right).
    \end{equation*}
\end{prop}
\begin{proof}
    Computing $\frac{d}{dt} f_t$ yields
    \begin{equation*}
        \frac{d}{dt} f_t(x) = - \bar{\Theta}_{\infty}^{(L)}(x^\top X^\top) \exp\left( -t \bar{\Theta}_\infty^{(L)} \left( X X^\top \right) \right) \left( f_0(X) - f^*(X) \right).
    \end{equation*}
    which satisfies the differential equation of the gradient flow problem in \citet{jacot2018neural} with initial condition $f_0(x)$. Moreover, it agrees at $t=\infty$ with Proposition~\ref{prop:jacot_f_infty}. To obtain the eigenvalues, we write
    \begin{align*}
        \left( \bar{\Theta}^{(L)} \left( X X^\top \right) \right)^{-1} \left( \exp \left( -t \bar{\Theta}_{\infty}^{(L)} \left( X X^\top \right) \right) - I_n \right)&=
        V^{(L)} \Lambda_{\text{inv}}^{(L)} V^{(L)\top} V^{(L)} 
        \left( \exp\left(-t \Lambda ^{(L)}\right) - I_n \right)V^{(L)\top} \\
        &= V^{(L)} \Lambda_{\text{inv}}^{(L)} \left( \exp \left(  -t  \Lambda^{(L)} \right)  - I_n\right)V^{(L)\top},
    \end{align*}
    where $\Lambda_{\text{inv}}^{(L)}$ is the inverse of $\Lambda^{(L)}$. The statement follows since $\Lambda^{(L)}$ is a diagonal matrix.
\end{proof} 
We call the term $\bar{\Theta}_{\infty}^{(L)} \left( x^\top X^\top \right) \left( \bar{\Theta}_\infty^{(L)} \left( X X^\top \right) \right)^{-1} \left( \exp \left( -t \bar{\Theta}_\infty^{(L)} \left( X X^\top \right) \right) - I_n \right)$ the \textbf{projection operator}. When necessary to highlight the dependence on $x$, we will instead refer to the \textbf{projection operator induced by $x$}.

\begin{rem}
    In Proposition~\ref{prop:matrix_exp}, we depart from the conventions of \citet{jacot2018neural} by omitting the explicit dataset scale factor $\frac{1}{n}$ within the exponential mapping. This structural choice directly impacts the asymptotic rates established in Theorem~\ref{thm:deep_stable_ntk}, since the eigenvalues $\lambda_i$ are now mapped to $\frac{e^{-\frac{t \lambda_i}{n}} - 1}{\lambda_i}$. Note that other ``learning rates'' can be used via a hyperparameter $\gamma$ (see \citet{lee2019wide} for a usage example). 
\end{rem}

\begin{prop}\label{prop:unfiform_close}
    Let $n_0 \geq 3$. If data points are sampled uniformly from $S^{n_0 - 1}$, then for any $1 > \delta > 0$, the probability that there is $x, x' \in S^{n_0-1}$ such that $x^\top x' \in [1 - \delta, 1]$ is upper bounded by
    \begin{equation*}
        \binom{n}{2} \frac{1}{\sqrt{\pi}} \sqrt{\frac{n_0}{2}} \frac{\delta^{n_0 - 1}}{n_0 - 1}.
    \end{equation*}
\end{prop}
\begin{proof}
    Using a known reult in \citet{cai2013distributions}, the probability density of the angle between two uniformly sampled points in $S^{n_0 -1}$ is given by
    \begin{equation*}
        P(\theta \leq \delta) = \frac{1}{\sqrt{\pi}} \frac{\Gamma\left( \frac{n_0}{2} \right)}{\Gamma\left( \frac{n_0 - 1}{2} \right)} \int_0^\delta \sin^{n_0 - 2}(\theta) \, d \theta,
    \end{equation*}
    and via a union bound argument, we obtain
    \begin{align*}
        P(\exists x,x' \in X \mid x^\top x' \in [1-\delta, 1]) &= \binom{n}{2} \frac{1}{\sqrt{\pi}} \frac{\Gamma\left( \frac{n_0}{2} \right)}{\Gamma\left( \frac{n_0 - 1}{2} \right)} \int_0^\delta \sin^{n_0 - 2}(\theta) \, d \theta \\
        &\leq \binom{n}{2} \frac{1}{\sqrt{\pi}} \sqrt{\frac{n_0}{2}} \int_0^\delta \sin^{n_0 -2}(\theta) \, d \theta \tag{$\frac{\Gamma(x + 1)}{\Gamma(x + \frac{1}{2})} \leq \sqrt{x+1} $ for $x > 0$} \\
        &\leq \binom{n}{2} \frac{1}{\sqrt{\pi}} \sqrt{\frac{n_0}{2}} \int_0^\delta \theta^{n_0 - 2} \, d \theta \tag{$\sin(\theta) \leq \theta$ for $\theta < 1$}\\
        &= \binom{n}{2} \frac{1}{\sqrt{\pi}} \sqrt{\frac{n_0}{2}} \frac{\delta^{n_0 -1}}{n_0 - 1}.
    \end{align*}
\end{proof}

Proposition~\ref{prop:unfiform_close} demonstrates that with high probability, distinct data points will be not be close in orientation. In conjunction with Theorem~\ref{thm:deep_stable_ntk}, we recover stability in many experimental scenarios. To see this, we refer to Example~\ref{ex:deep_stable}.

\begin{example}\label{ex:deep_stable}
    For $n_0 = 128$ and $n = 2^{30}$, we set $\delta = \binom{n}{2}^{-\frac{1}{n_0 -1}}$ in Proposition~\ref{prop:unfiform_close}. Evaluating $\delta \approx 0.72$, we have that $P(\exists x, x' \in X \mid x^\top x' \in [1-\delta, 1]) \leq 0.04$. With high probability, the expression $\frac{1}{\sqrt{1 - \rho^{(2)}(x^\top x')}} \lessapprox 1.37$. The offset in the recurrence of Theorem~\ref{thm:deep_stable_ntk} is dominated by the depth $L$ for even small values of $L$. This shows that for a rather large dataset size $n$ and a much smaller input dimension $n_0$, we still obtain fast convergence of $\Theta_{\infty}^{(L)} \left( X X^\top \right)$ to $\frac{1}{4} \mathbf{1}_n \mathbf{1}_n^\top + \frac{3}{4} I_n$ for small values of $L$. The degenerate case where some $x, x' \in X$ are arbitrarily close to each other can be excluded with high probability.
\end{example}

In fact, Theorem 2 from \citet{cai2013distributions} shows that the probability that $x_i^\top x_j \in \left[1 - n^{-\frac{2}{n_0 - 1}}, 1 \right]$, where $i,j \in [n]$ and $i \neq j$, is upper bounded by $1 - e^{-K}$ for some constant $K$ defined as
\begin{equation*}
    K := \frac{1}{4 \sqrt{\pi}} \frac{\Gamma\left( \frac{n_0}{2} \right)}{\Gamma \left( \frac{n_0 + 1}{2} \right)}.
\end{equation*}
In Example~\ref{ex:deep_stable}, this upper bound translates to a bound $\approx 0.02$. Note that the probability does not depend on $n$. By scaling $L$ with, say, $\frac{1}{\sqrt{1 - h\left(1 - \delta \right)}} \approx \frac{1}{\sqrt{\delta}} = \binom{n}{2}^{\frac{1}{2(n_0 -1)}}$, with a $\approx 98\%$ chance, we will have a stable output $f_t$ that retains its ability to distinguish different inputs.

\subsection{Proof of Theorem~\ref{thm:convergence_theta_bar}}\label{sec:proof_convergence_theta_bar}

To prove Theorem~\ref{thm:convergence_theta_bar}, we use definitions and results from Section~\ref{sec:auxiliary_results}, namely Lemma~\ref{lem:convergence_rho}, Definition~\ref{def:hfun} and Proposition~\ref{prop:theta_recursive}. The proof strategy relies on writing the kernel $\Theta_{\infty}^{(L)}$ using the convex combination of Proposition~\ref{prop:theta_recursive}, and expanding recursively for superscripts $(l)$. The convergence property of $\rho^{(L)}$ in Lemma~\ref{lem:convergence_rho} ensures that the expansion converges. 

\begin{proof}
    Let $z \in [-1,1]$. If we define the sequence of error terms $\{e_{L}(z) \}_{L=1}^\infty$ from 1 at layer $L$ via
    \begin{align*}
        e_{L+1}(z) &= 1 - h\left(\rho^{(L)}(z)\right) \\
        e_2(z) &= 1 - h(\rho^{(1)}(z)) = 1 - \rho^{(2)}(z),
    \end{align*}
    we use the Taylor expansion of $\sqrt{1 - \frac{z}{2}}$ and $\arccos(1-z)$ at $z=0$ to obtain
    \begin{align*}
        \sqrt{1-(1-z)^2} &= \sqrt{2z - z^2} = \sqrt{2z} \sqrt{1 - \frac{z}{2}} \\
        &= \sqrt{2z} \left( 1 - \frac{z}{4} + O\left( z^{2} \right) \right) \\
        &= \sqrt{2}z^{\frac{1}{2}} - \frac{\sqrt{2} z^{\frac{3}{2}}}{4} + O \left( z^{\frac{5}{2}} \right) \\
        \arcsin(1 - z) &= \frac{\pi}{2} - \arccos(1 - z) \\
        &= \frac{\pi}{2} - \sqrt{2} z^{\frac{1}{2}} + \frac{\sqrt{2}}{12} z^{\frac{3}{2}} + O\left( z^{\frac{5}{2}} \right).
    \end{align*}
    These results imply that for $z \approx 0$,
    \begin{align*}
        h(1 - z) &= \frac{1}{\pi} \left( \frac{\pi}{2} - \sqrt{2} z^{\frac{1}{2}} - \frac{\pi}{2} z + \frac{11 \sqrt{2}}{12} z^{\frac{3}{2}} + \sqrt{2} z^{\frac{1}{2}} - \frac{\sqrt{2}}{4} z^{\frac{3}{2}}  \right) + \frac{1 - z}{2} + O\left( z^{2} \right) \\
        &= 1 - z + \frac{2 \sqrt{2}}{3 \pi} z^{\frac{3}{2}} + O\left( z^2 \right).
    \end{align*}
    We obtain the recurrence relation
    \begin{align*}
        &\phantom{{}\leq{}} \\
        && 1 - e_{L+1}(z) &= h\left( 1 - \left(1 - \rho^{(L)}(z) \right) \right) = h\left( 1 - e_L(z) \right) \\
        &&&= 1 - e_L(z) + \frac{2\sqrt{2}}{3 \pi} e_L^{\frac{3}{2}}(z) + O \left( e_L^2(z) \right) \\
        &\implies &e_{L+1}(z) &= e_L(z) - \frac{2 \sqrt{2}}{3 \pi} e_L^{\frac{3}{2}}(z) + O \left( e_L^2 (z)\right).
    \end{align*}
    If we define $V_L(z) := e_L^{-\frac{1}{2}}(z)$, using the Taylor expansion of $(1 - x)^{-\frac{1}{2}} = 1 + \frac{x}{2} + O \left( x^2 \right)$ at $x=0$, we obtain
    \begin{align*}
        V_{L+1}(z) &= e^{-\frac{1}{2}}_{L+1}(z) = \left( e_L(z) - \frac{2 \sqrt{2}}{3 \pi} e_L^{\frac{3}{2}}(z) + O \left( e_L^2(z)\right) \right)^{-\frac{1}{2}} \\
        &= e_L^{-\frac{1}{2}}(z) \left( 1 - \frac{2 \sqrt{2}}{3 \pi} e_L^{\frac{1}{2}}(z) + O \left( e_L(z) \right) \right)^{-\frac{1}{2}} \\
        &= e_L^{-\frac{1}{2}}(z) \left( 1 + \frac{\sqrt{2}}{3 \pi} e_L^{\frac{1}{2}}(z) + O\left( e_L(z) \right) \right) \\
        &= V_L(z) + \frac{\sqrt{2}}{3 \pi} + O \left( e_L^{\frac{1}{2}}(z) \right).
    \end{align*}
    Using a telescoping sum, the difference $V_{L+1} - V_2$ can be written as
    \begin{align*}
        V_{L+1}(z) - V_2(z) &= \sum_{k=2}^{L} \left( V_{k+1}(z) - V_{k}(z) \right)\\
        &=  \frac{L\sqrt{2}}{3 \pi} + \sum_{k=2}^{L} O \left( e_k^{\frac{1}{2}}(z) \right).
    \end{align*}
    By the Stolz-Cesàro Theorem, 
    \begin{equation*}
        \lim_{L \to \infty} \frac{\sum_{k=2}^{L} O \left( e_k^{\frac{1}{2}}(z) \right)}{L} = \lim_{L \to \infty} O\left( e_{L+1}^{\frac{1}{2}}(z) \right) = 0.
    \end{equation*}
    This implies that $V_{L+1}(z) - V_2(z) \in \mathcal{O}\left( \frac{L \sqrt{2}}{3 \pi} \right)$, from which we deduce that $e_L(z) \in \mathcal{O} \left( \frac{1}{L^2} \right)$. Now, to find the error $1 - \bar{\Theta}^{(L)}_\infty (x^\top x)$, we use the Taylor expansion of $h'(1 - z)$ when $z \approx 0$:
    \begin{align*}
        h'(1-z) &= 1 - \frac{\sqrt{2}}{\pi} z^{\frac{1}{2}} + O\left( z^{\frac{3}{2}}\right)
    \end{align*}
    Substituting $e_{L}(z)$ for $z$ in the equation above, we get
    \begin{align*}
        1 - h'(1 - e_L(z)) &= 1 - \left(1 - \frac{\sqrt{2}}{\pi} e_{L}^{\frac{1}{2}}(z) + O\left( e_{L}^{\frac{3}{2}}(z)\right) \right) \\
        &= \frac{\sqrt{2}}{\pi} e_{L}^{\frac{1}{2}}(z) + O \left( e_{L}^{\frac{3}{2}}(z) \right) \\
        &\in \mathcal{O} \left( \frac{1}{L} \right).
    \end{align*}

We can now express the convergence rate of $\Theta_{\infty}^{(L)}(x, x')$ for $x^\top x' \in [-1, 1[$. To simplify notation, we use the definition of $\epsilon_L(z)$ to refer to $\Theta_{\infty}^{(L)}(x, x')$ when $x^\top x' = z$. Via Proposition~\ref{prop:theta_recursive}, we obtain the recurrence relation
\begin{align*}
    \epsilon_{L+1}(z) &= \frac{L}{L+1} h'(1 - e_{L}(z)) \epsilon_L(z) + \frac{1}{L+1} (1 - e_{L+1}(z)) \\
    &= \frac{L}{L+1} \left(1 - \frac{\sqrt{2}}{\pi} e_L^{\frac{1}{2}}(z) + O \left( e_L^{\frac{3}{2}}(z)\right) \right) \epsilon_L(z) \\
    &+ \frac{1}{L+1} \left( 1 - \left(e_1^{\frac{1}{2}}(z) +  \frac{L\sqrt{2}}{3 \pi} + O \left( \sum_{k=1}^L e_k^{\frac{1}{2}}(z)  \right)\right)^{-2} \right) \\
    &= \frac{L}{L+1} \left( 1 - \frac{3}{L} \right) \epsilon_L(z) + \frac{1}{L+1} + o \left( \frac{1}{L} \right) \\
    &= \left( 1 - \frac{4}{L+1} \right) \epsilon_L(z) + \frac{1}{L+1} + o \left( \frac{1}{L} \right).
\end{align*}
This can be rewritten to obtain
\begin{align}\label{eq:epsilon_recurrence}
    \epsilon_{L+1}(z) &= \left( 1 - \frac{4}{L+1} \right) \epsilon_L(z) + \frac{1}{L+1} + o \left( \frac{1}{L} \right).
\end{align}
Ignoring the $o \left( \frac{1}{L} \right)$, equation~\eqref{eq:epsilon_recurrence} has solution $\frac{1}{4}$. To see this, we expand~\eqref{eq:epsilon_recurrence} up to $\epsilon_l(z)$, for a fixed and large $l$, to obtain
\begin{align*}
    \epsilon_{L+1}(z) &= \sum_{k=2}^{L+1} \frac{1}{k} \prod_{k'=k+1}^{L+1} \left( 1 - \frac{4}{k'}\right)  + o \left( \frac{1}{L} \right) 
    + \epsilon_2(z) \prod_{k'=2}^{L+1} \left( 1 - \frac{4}{k'} \right)\\
    &= \sum_{k=2}^{L+1} \frac{1}{k} \exp\left( \sum_{k'=k+1}^{L+1} \ln \left( 1 - \frac{4}{k'} \right) \right) + o \left( \frac{1}{L} \right) 
    + \epsilon_2(z) \exp\left( \sum_{k'=2}^{L+1} \ln\left( 1 - \frac{4}{k'} \right) \right) \\
    &\leq \sum_{k=2}^{L+1} \frac{1}{k} \exp \left(- \sum_{k'=k+1}^{L+1} \frac{4}{k'} \right) + o \left( \frac{1}{L} \right) \tag{$\ln(1 - x) \leq - x$ for $0 < x \leq 1$} 
    + \epsilon_2(z) \exp\left( \sum_{k'=2}^{L+1} \left(- \frac{4}{k'} \right) \right) \\
    &\leq \sum_{k=2}^{L+1} \frac{1}{k} \exp \left( 4 \left( \ln\left(\frac{k}{L+1}\right)  + \delta \right) \right) \tag{using $\ln(n) \approx \sum_{k=1}^n \frac{1}{k} + \gamma$ for large $n$}  + o \left( \frac{1}{L} \right) 
    + \epsilon_2(z) \exp\left( 4 \left( \ln \left(\frac{1}{L+1} \right) + \delta \right) \right) \\
    &= e^\delta \sum_{k=2}^{L+1} \frac{1}{k} \left( \frac{k}{L+1} \right)^4 + o \left( \frac{1}{L} \right) 
    + e^\delta \epsilon_2(z)  \left(\frac{1}{L+1} \right)^4,
\end{align*}
where the multiplication factor $e^\delta$ can be made close to $1$ (i.e. $\delta \to 0^+ $ by choosing a large enough $L$. Since the sum
\begin{equation*}
    \sum_{k=1}^{L+1} k^3 = \left( \frac{(L+1)(L+2)}{2} \right)^2 \tag{sum of first $L+1$ cubes formula},
\end{equation*}
we can expand the right-hand side to obtain $\frac{(L+1)^4 + 2(L+1)^3 + (L+1)^2}{4}$. Dividing by $(L+1)^4$, we finally obtain $\frac{1}{4} + \frac{1}{2(L+1)} + \frac{1}{4 (L+1)^2}$. As $L \to \infty$, this approaches $\frac{1}{4}$.
\end{proof}

\begin{proof}[Corollary~\ref{cor:convergence_theta_bar}]
    In the proof of Theorem~\ref{thm:convergence_theta_bar}, the leading term in the error $1 - h'(1 - e_L(z))$ is $\frac{\sqrt{2}}{\pi} e_L^{\frac{1}{2}}(z)$. Notice that
    \begin{equation*}
        e_L^{\frac{1}{2}}(z) = \frac{1}{e_2^{-\frac{1}{2}}(z) + \frac{L\sqrt{2}}{\pi} + o(L)}.
    \end{equation*}
    The claim immediately follows from $e_2^{-\frac{1}{2}}(z) = \frac{1}{\sqrt{1 - h(\rho^{(1)}(x^\top x'))}} = \frac{1}{\sqrt{1 - \rho^{(2)}(x^\top x')}}$.
\end{proof}

\subsection{Proof of Theorem~\ref{thm:deep_stable_ntk}}

By Theorem~\ref{thm:convergence_theta_bar}, we see that $\bar{\Theta}_\infty^{(L)}(x, x') \to \frac{1}{4}$ for $x, x' \in S^{n_0 - 1}$ with $x \neq x'$ (i.e., off-diagonal components); the convergence rate is $O\left( \frac{1}{L} \right)$. Notice that that there is a hidden dependence on the inputs $x$ and $x'$ inside the $O$ notation. However, as $L$ grows large, we can safely expect that the convergence will be upper bounded by a $\frac{1}{L}$ decay. Moreover, since $h$ and $h'$ are strictly increasing on $[-1, 1]$, it is easy to observe that $\bar{\Theta}_{\infty}^{(L)}(x, x')$ is increasing in $x^\top x'$ for any $L$. Since the input dimension $n_0$ is fixed, as we increase the number of data points in $X$, e.g. through data collection, we will inevitably get points that are less separated (i.e. closer to each other). This will drive the smallest eigenvalue of $\bar{\Theta}_{\infty}^{(L)} \left( X X^\top \right)$ towards $0$ (when $L$ is fixed). Meanwhile, the component-wise convergence to $\frac{1}{4}$ does not depend on the size $n$ of the dataset, and only depends on $L$. As we show next, it is possible to control the stability of the learning dynamics of $f_t$ by properly scaling the depth $L$ and performing early stopping.

By inspection of the projection operator of Proposition~\ref{prop:matrix_exp}, we observe that the matrix
    \begin{equation}\label{eq:mapped_eigvals}
        \left( \bar{\Theta}_{\infty}^{(L)} \left( X X^\top \right) \right)^{-1} \left( \exp \left( -t \bar{\Theta}_{\infty}^{(L)} \left( X X^\top \right) \right) -I_n  \right),
    \end{equation}
has eigenvalues $\frac{e^{-t \lambda_k} - 1}{\lambda_k}$. If $\{v_1, \dots, v_n \}$ is the set of eigenvectors with eigenvalues $\{\lambda_1, \dots, \lambda_n\}$ in the eigendecomposition of $\bar{\Theta}_{\infty}^{(L)} \left( X X^\top \right)$, the Davis-Kahan Theorem implies that the space spanned by $ \{v_2, \dots, v_n \}$ will be close in orientation to the space spanned by $\{\tilde{v}_2, \dots, \tilde{v}_n \}$, where $\tilde{v}_i$ are the eigenvectors of $\frac{1}{4} \mathbf{1}_n \mathbf{1}_n^\top + \frac{3}{4} I_n$ with eigenvalues $\{\lambda_1, \dots, \lambda_n \}$~\citep{tran2026davis}.
    \begin{align}\label{eq:davis-kahan-alignment}
        \left\lVert \sum_{i=2}^n v_i v_i^\top - \sum_{i=2}^n \tilde{v}_i \tilde{v}_i^\top \right\rVert_2 \leq  \frac{\pi \frac{n}{L}}{\frac{n}{4}} = \frac{4\pi}{L}
    \end{align}
    Therefore, for large enough $L$, the projection of any vector onto $\operatorname{span} \left( \{v_2, \dots, v_n \} \right)$ will be close to the projection onto $\operatorname{span} \left( \{\tilde{v}_2, \dots, \tilde{v}_n \} \right)$. 

The matrix in~\eqref{eq:mapped_eigvals} maps $\lambda_n$ to $\frac{e^{-t \lambda_n} - 1}{\lambda_n}$, and this is approximately $-t$ when $t \lambda_n $ is small. In this case, the evolution of $f_t$ is close to an approximation $\hat{f}_t$ defined as in Proposition~\ref{prop:matrix_exp} (see expression for $f_t$) with the matrix $\frac{1}{4} \mathbf{1}_n \mathbf{1}_n^\top + \frac{3}{4} I_n$ in place of $\bar{\Theta}_{\infty}^{(L)} \left( X X^\top \right)$. Equation~\ref{eq:davis-kahan-alignment} implies that the difference in norm between the projection of a vector on $\{v_2, \dots, v_n\}$ and its projection on $\{ \tilde{v}_2, \dots, \tilde{v}_n \}$ is in $O\left(\frac{1}{L}\right)$. For eigenvectors $v_1$ and $\tilde{v}_1$, a similar application of the Davis-Kahan Theorem yields $v_1^\top \tilde{v}_1 \approx 1 - \frac{4 \pi}{L} $ for large depth~\footnote{Note that $v_1^\top \tilde{v}_1 \geq 0$ for large $L$.}:
     \begin{align*}
         \left\lVert v_1 v_1^\top v_1 - \tilde{v}_1 \tilde{v}_1^\top v_1 \right\rVert_2^2 &= \left\lVert v_1 - \tilde{v}_1 (\tilde{v}_1^\top v_1) \right\rVert_2^2 \\
         &= \lVert v_1 \rVert_2^2 + (\tilde{v}_1^\top v_1)^2 \lVert \tilde{v}_1 \rVert_2 - 2 (\tilde{v}_1^\top v_1)^2 \\
         &= 1 - (\tilde{v}_1^\top v_1)^2 \\
         &\leq \left( \frac{4 \pi}{L} \right)^2.
     \end{align*}

\begin{namedproof}{a)}
    From Theorem~\ref{thm:convergence_theta_bar} and the fact that $x_i^\top x_j \leq 1 - \delta$ when $i \neq j$, it can be seen that $X$ satisfies $\mathcal{O}\left( \frac{1}{L} \right)$-separability in \citet{karhadkar2024bounds}. Lemma 7 from \citet{karhadkar2024bounds} shows that the smallest eigenvalue of the kernel $\Sigma^{(L)} \left( X X^\top \right)$ satisfies
    \begin{align*}
        \lambda_n \left( \bar{\Theta}_{\infty}^{(L+1)} \left( X X^\top \right) \right)
        &= \frac{L}{L+1} h'\left( \rho^{(L)} \left( X X^\top \right) \right) \bar{\Theta}_{\infty}^{(L)} \left( X X^\top \right)
        + \frac{1}{L+1} \rho^{(L+1)} \left( X X^\top \right) \\
        &\geq \frac{1}{L+1} \sum_{l=2}^{L}
        \lambda_n \left( h'\left( \rho^{(l)} \left( X X^\top \right) \right) \right) \\
        &\in \frac{1}{L+1} \sum_{l=2}^L \Omega \left( \left( 1 + \frac{n_0 \ln\left( L \right)}{\ln(n_0)} \right)^{-3} \frac{1}{l^2} \right),
    \end{align*}
    where the first inequality uses $\lambda_n(A + B) \geq \lambda_n(A) + \lambda_n(B)$ and $\lambda_n ( A \odot B ) \geq \lambda_n(A) \min_{i \in [n]} B_{ii}$ for positive semidefinite matrices $A, B$. The last line follows from Lemma 7 by \citet{karhadkar2024bounds} since the $x_i \in X$ can be mapped to $x'_i \in S^{n_0-1}$ such that $(x'_i)^\top x'_j = \rho^{(l)}(x_i^\top x_j)$ for $l \geq 2$.  Hiding the dependence on $n_0$, this lower bound can be re-written as $\Omega\left( \frac{1}{L} \right)$. This implies that as $t \to \infty$, the norm $\lVert f_t(x) - \hat{f}_t(x) \rVert_2 \in O(n)$ for any $x \in S^{n_0 -1 }$, where the $O$ notation hides the dependence on $n_0$ and $f^*$. 
    
    Using the fact that $\lVert A(x - x') \rVert_2 \geq \min_{i \in [n]} |\lambda_i(A)| \times\lVert x - x' \rVert_2$ for an $n\times n$ matrix $A$ that is symmetric and diagonalizable, we get
    \begin{align}\label{eq:error_with_depth}
        &\phantom{{}\geq{}}
        \left\lVert 
        \left( \bar{\Theta}_{\infty}^{(L)} \left( X X^\top \right) \right)^{-1} \left( \exp \left( -t \bar{\Theta}_{\infty}^{(L)} \left( X X^\top \right) \right) - I_n \right)
        \left( \bar{\Theta}_{\infty}^{(L)} \left( X x \right) - \bar{\Theta}_{\infty}^{(L)} \left( X x' \right) \right) \right\rVert_2 \\
        &\geq
        \min_{i \in [n]} \frac{| e^{-t \lambda_i} - 1 | }{\lambda_i} \left\lVert \bar{\Theta}_{\infty}^{(L)} \left( X x \right) - \bar{\Theta}_{\infty}^{(L)} \left( X x' \right) \right\rVert_2 \nonumber \\
        &\in \Omega\left( \frac{1}{n} \left\lVert \bar{\Theta}_{\infty}^{(L)} \left( X x \right) - \bar{\Theta}_{\infty}^{(L)} \left( X x' \right) \right\rVert_2 \right) \nonumber \\
        &= \Omega \left( \frac{1}{L} \right) \nonumber,
    \end{align}
    where the notation $\Omega\left( \frac{1}{L} \right)$ hides the dependence on $x, x'$. As shown in Corollary~\ref{cor:convergence_theta_bar}, this dependence on $x,x'$ is relevant when $L \in O\left( \frac{1}{\sqrt{1 - \rho^{(2)} \left(x^\top x' \right)}} \right)$. In this instance, the error term~\eqref{eq:error_with_depth} is bounded away from $0$ via $ 1 - \rho^{(2)}(x^\top x') \geq \delta$. Therefore, prediction $f_t(x)$ is not trivial whenever $L$ is chosen appropriately. For a dataset $X$ with separation condition $\delta$, this is equivalent to $L \in O\left( \frac{1}{\sqrt{\delta}} \right)$.
 \end{namedproof}
 
 \begin{namedproof}{b)}
    In the case where $\lambda_n$ can go to $0$, it remains possible to control the effect of $\lambda_n$ on $f_t$. The matrix in~\eqref{eq:mapped_eigvals}
    has eigenvalues $\frac{e^{-t \lambda_k} - 1}{\lambda_k} \approx -t ( 1 - \frac{t}{2} \lambda_k)$ when $t \lambda_k \approx 0$. If $\lambda_n \to 0$ as $n$ increases, we can choose $t$ such that $0 < t \lambda \ll 1$. This ensures that the norm of the projection operator induced by $x$ is $O(n)$. Note that $\lambda_1$ is mapped to $\approx 0$ for large $L$ or $n$.

    As in a), $f_t$ remains non-trivial as long as $L$ is dominated by the separation of datapoints: $L \in O \left( \max_{i,j \in [n] \mid i \neq j} \frac{1}{\sqrt{1- \rho^{(2)} \left( x_i^\top x_j 
    \right)}} \right)$.
 \end{namedproof}

\subsection{Experiments and theoretical implications}\label{sec:experiments}

In the proof of Theorem~\ref{thm:deep_stable_ntk} from the previous section, we can identify the key properties used to derive the results. This allow us to distill the essential characteristics of sequences of general kernels that exhibit similar limiting behaviour.

For an arbitrary sequence of kernels $\kappa^{(L)}$ to satisfy the structural conditions used in the proof of the theorem, we first require that $\kappa^{(L)}(x, x) \geq \kappa^{(L)}(x_1, x_2)$ for any inputs $x, x_1, x_2 \in S^{n_0-1}$ across all depths $L \in \mathbb{N}$. That is, the diagonal terms (i.e., colinear data points) must be the largest values. As a second requirement, there must be some depth $\hat{L} \in \mathbb{N}$ such that $\kappa^{(L)}\left(X X^\top\right)$ is positive definite for any dataset $X = \left\{x_i \mid x_i \in S^{n_0 - 1}, i = 1, \dots, n \right\}$ and $L\geq \hat{L}$. In other words, once the network is sufficiently deep, its corresponding limiting kernel is positive definite. 

By inspection of Definition~\ref{def:correl_coeff}, it can be observed that $\Theta_{\infty}^{(L)} \left( X X^\top \right)$ with $\hat{L} = 2$ satisfies the criteria of the list above. Another example is given by the sequence $\eta^{(L)}$ of kernels that are defined recursively by $\eta^{(L+1)}(x, x') = h(\kappa^{(L)}(x,x'))$ and $\eta^{(1)}(x, x') = x^\top x'$, where $h(z)= \left(1 + e^{-z} \right)^{-2}$ (see Proposition~\ref{prop:convergence_eta} in Appendix~\ref{sec:additional_lemmas}).

In order to better understand the theoretical insights from the previous section, we empirically evaluate the convergence rates of $\bar{\Theta}_{\infty}^{(L)}$, $\rho^{(L)}$ and $\eta$ as $L$ increases. We illustrate this convergent behavior in Figure~\ref{fig:convergence_rates}, where we generate a dataset $X$ and a point $x$ from the uniform distribution ($n_0 = 128$) and we canonically project them onto the sphere. We then plot the evolution of these values for depths $L=1,\dots, 30$. Note that this depth limit is sufficient to show convergence. Each curve in the plots corresponds to a different pair of inputs (either from $X $ or between $x$ and $X$; e.g., $\bar{\Theta}_{\infty}^{(L)}(x_i, x_j)$). It is immediately apparent that both $\rho^{(L)}$ and $\eta^{(L)}$ converge to a fixed value rather quickly as the depth increases. 

\begin{figure}[ht]
    \centering
    \includegraphics[scale=0.83]{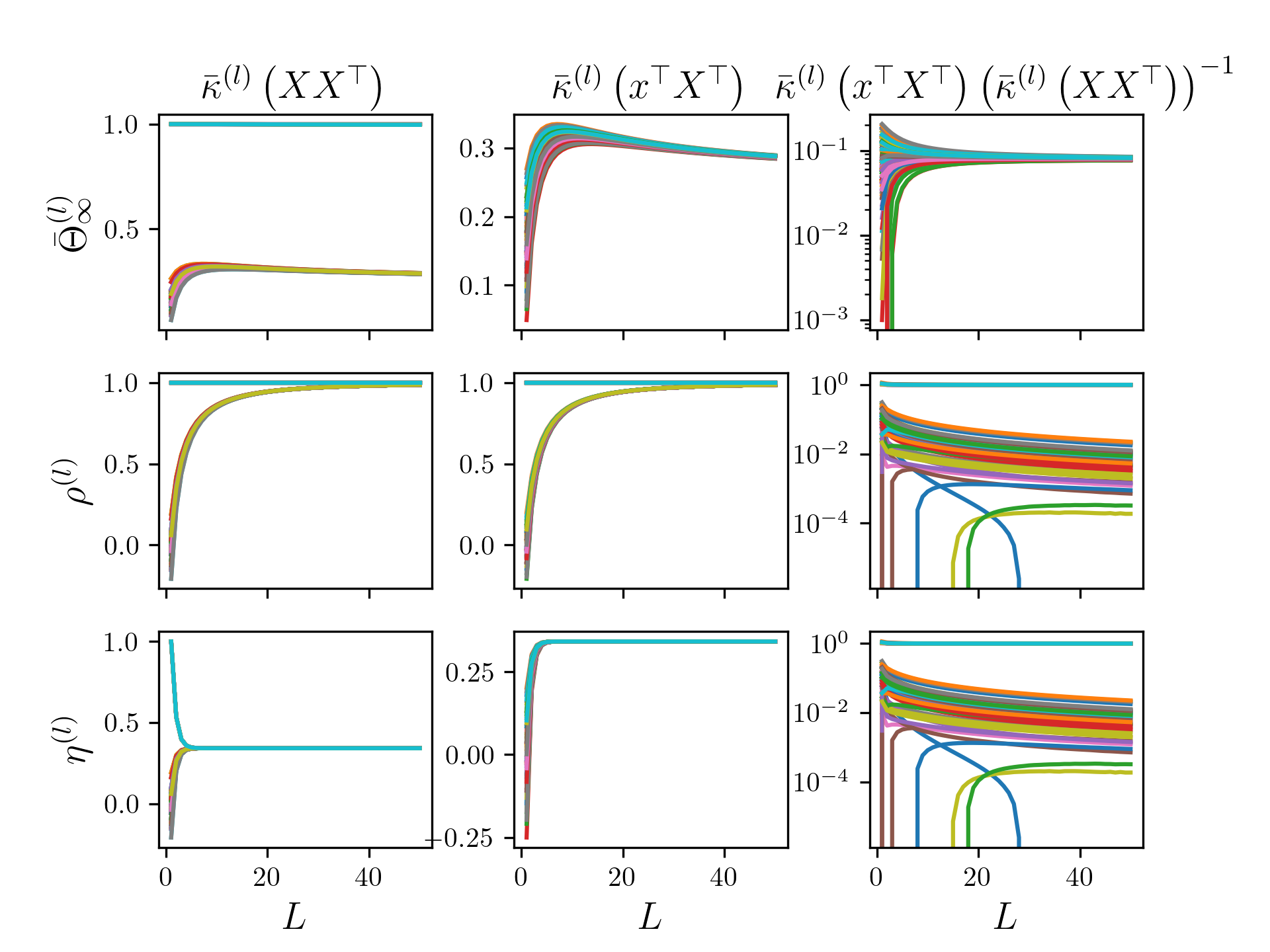}
    \caption{\textbf{Synthetic} dataset. Convergence rate of $\kappa$ on $X$ and point $x$. We model three different expressions dependent on an arbitrary $\kappa$ in each column. The particular choice of $\kappa$ is given by each row (see label on the left).}
    \label{fig:convergence_rates}
\end{figure}

In addition to the synthetic dataset $X$, we perform the same experiment on the MNIST dataset~\citep{lecun2010mnist}, whose tensors are converted to vectors that are normalized to lie on the sphere. Data points are chosen uniformly at random from the training set; we use sampling without replacement to guarantee that data points are not colinear. We report the results in Figure~\ref{fig:convergence_rates_mnist}.

\begin{figure}[ht]
    \centering
    \includegraphics[scale=0.83]{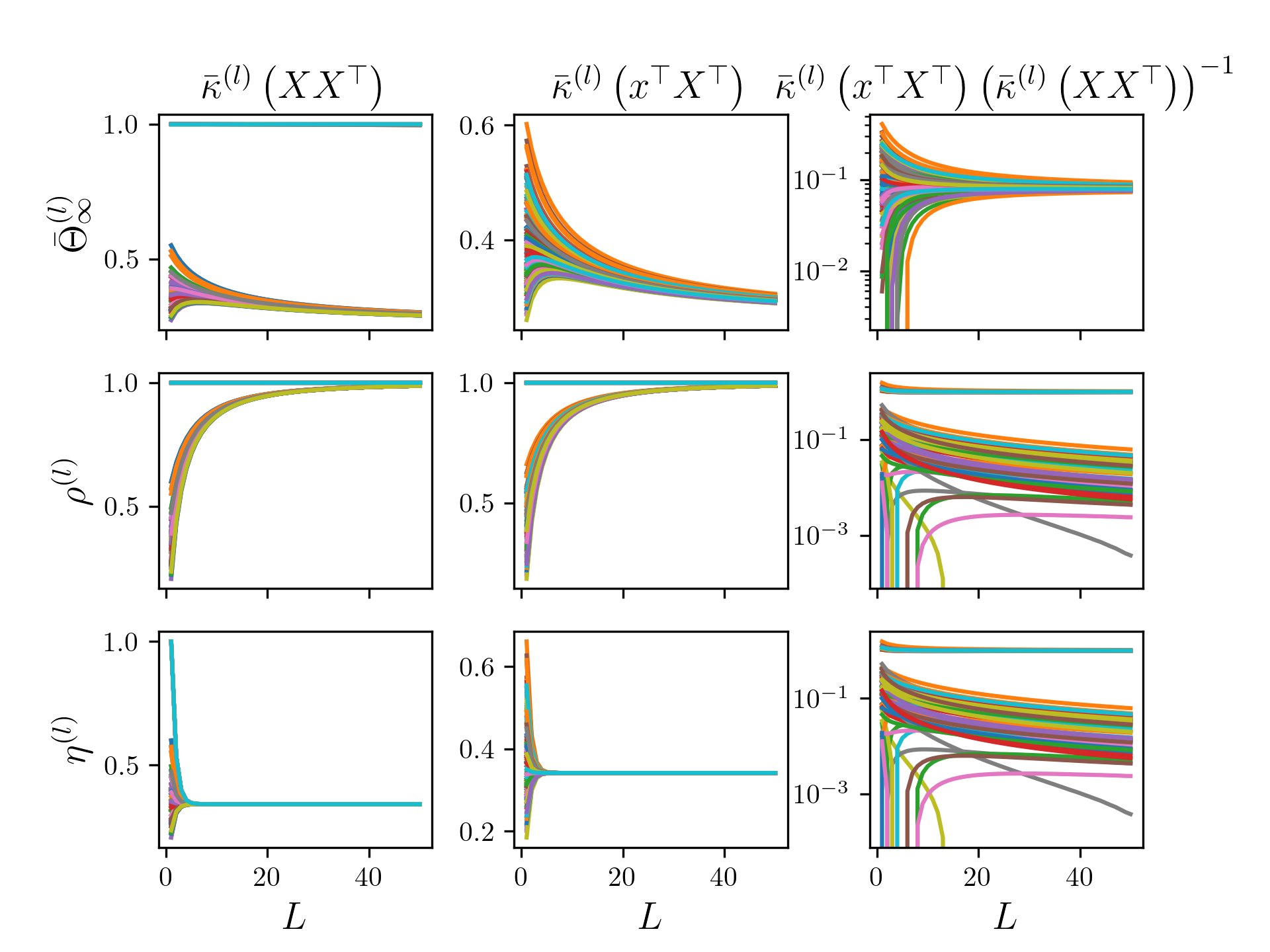}
    \caption{\textbf{MNIST} datset. Convergence rate of $\kappa$ on $X$ and point $x$. We model three different expressions dependent on an arbitrary $\kappa$ in each column. The particular choice of $\kappa$ is given by each row (see label on the left).}
    \label{fig:convergence_rates_mnist}
\end{figure}

Across both the synthetic and MNIST datasets, we evaluate the norm of the projection term in the expansion of the prediction function $f_t$. In order to control the size of decreasing eigenvalues, we use a bounded stopping time $t \in \mathbb{N} \cap \left[ 0, \left\lceil \frac{1}{\lambda_n} \right\rceil \right]$ and $L = \max_{i,j \in [n] \mid i \neq j} \frac{1}{\sqrt{1 - h(x_i^\top x_j)}}$. We plot the norms in Figure~\ref{fig:stable_norms}.

\begin{figure}
    \centering
    \includegraphics[width=0.9\linewidth]{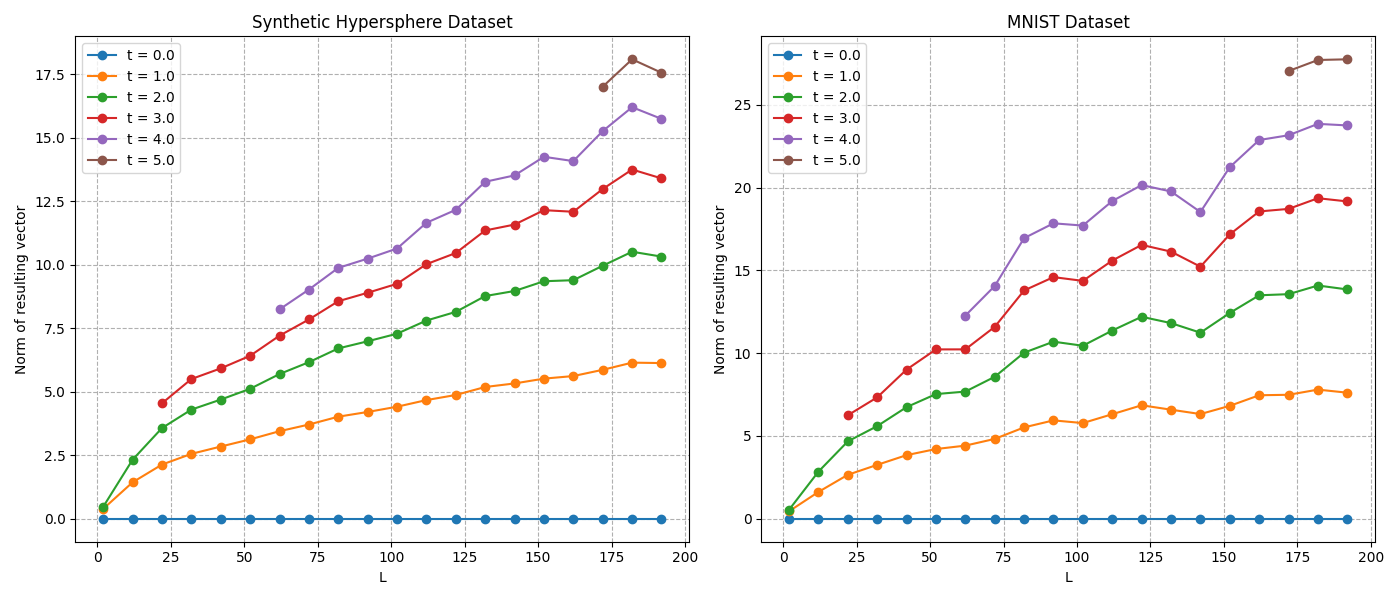}
    \caption{Norm of the projection term in the closed-form solution for $f_t(x)$ with respect to depth $L$. We use stopping times $t \in \mathbb{N} \cap \left[0, \left\lceil \frac{1}{\lambda_n} \right\rceil \right]$ and a depth of $L = \max_{i,j \in [n] \mid i \neq j} \frac{1}{\sqrt{1 - h(x_i^\top x_j)}}$ to show how the norm scales as we increase $n$.}
    \label{fig:stable_norms}
\end{figure}

In both datasets, the convergence rate is sublinear for $\bar{\Theta}_{\infty}^{(L)}\left( X X^\top \right)$ and Theorem~\ref{thm:convergence_theta_bar} shows that the off-diagonal component error is $O\left( \frac{1}{L} \right)$. The convergence rate to $\frac{1}{4}$ is logarithmic. As is demonstrated in Figure~\ref{fig:stable_norms}, it is possible to control the norm of the projection term using a proper stopping time and scaling the depth appropriately, as per Theorem~\ref{thm:deep_stable_ntk}. This reflects the fact that in practice, typical experiments perform early stopping, while depth is usually much smaller than the size of the dataset. These two characteristics in combination thus yield the ``Goldilocks'' zone where the NTK is stable, yet it retains its expressivity. To support this idea, we use the setting of Example~\ref{ex:deep_stable} with $n_0 = 128$ to measure the smallest eigenvalue of $\bar{\Theta}_{\infty} \left( X X^\top \right)$ in Figure~\ref{fig:smallest_eigval_scaling}, and track the bound on $L$ from Theorem~\ref{thm:deep_stable_ntk} in Figure~\ref{fig:scale_depth_uniform}. We observe that for $L \in \{1, \dots, 30\}$, the smallest eigenvalue increases with depth, even as we increase the dataset size $n$. For large values of $n$, a depth of 5 stabilizes the smallest eigenvalue, while guaranteeing that $f_t$ is expressive near $X$. While it remains possible to go beyond this depth, values of $L$ that are too large will result in data-independent predictions~\citep{xiao2020disentangling}.

\begin{figure}[htbp]
    \centering
    \includegraphics[width=0.95\linewidth]{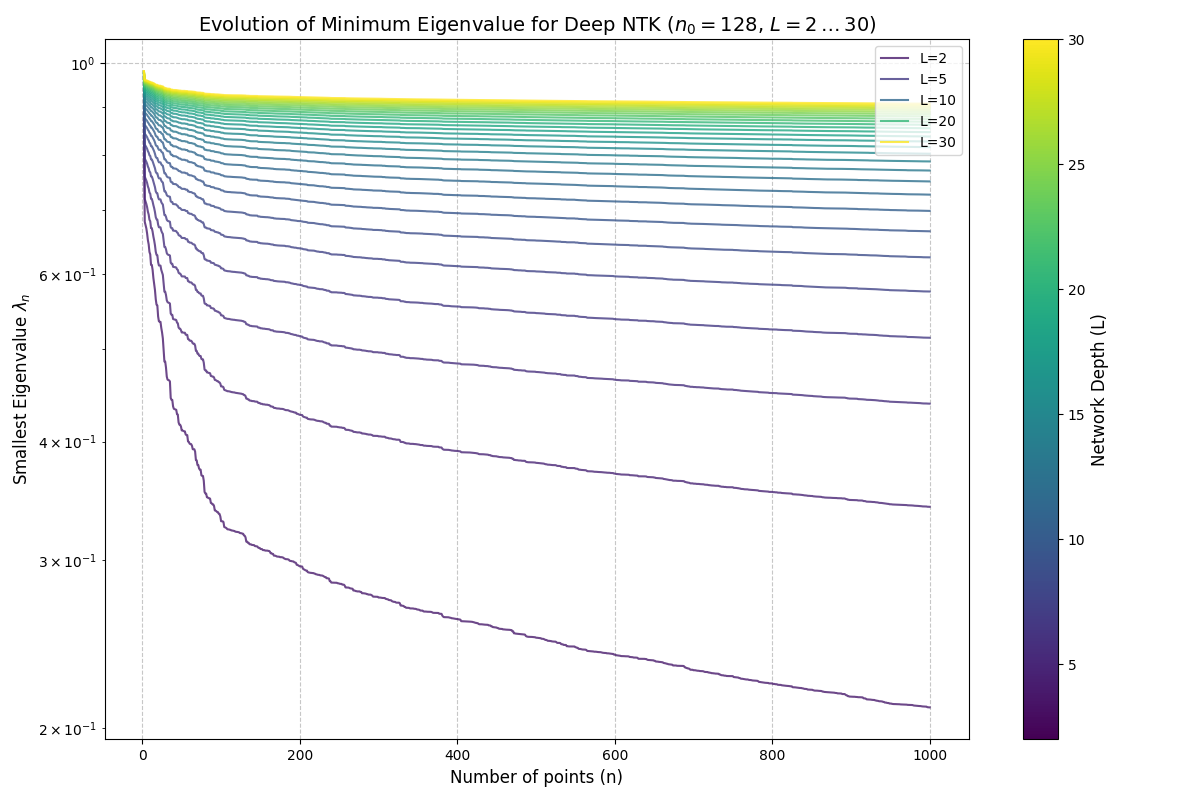}
    \caption{Evolution of the smallest eigenvalue of $\Theta_{\infty}^{(L)} \left( X X^\top \right)$ for varying depths $L$ and dataset size $n$. We used $n_0 = 128$.}
    \label{fig:smallest_eigval_scaling}
\end{figure}

\begin{figure}[htbp]
    \centering
    \includegraphics[width=0.95\linewidth]{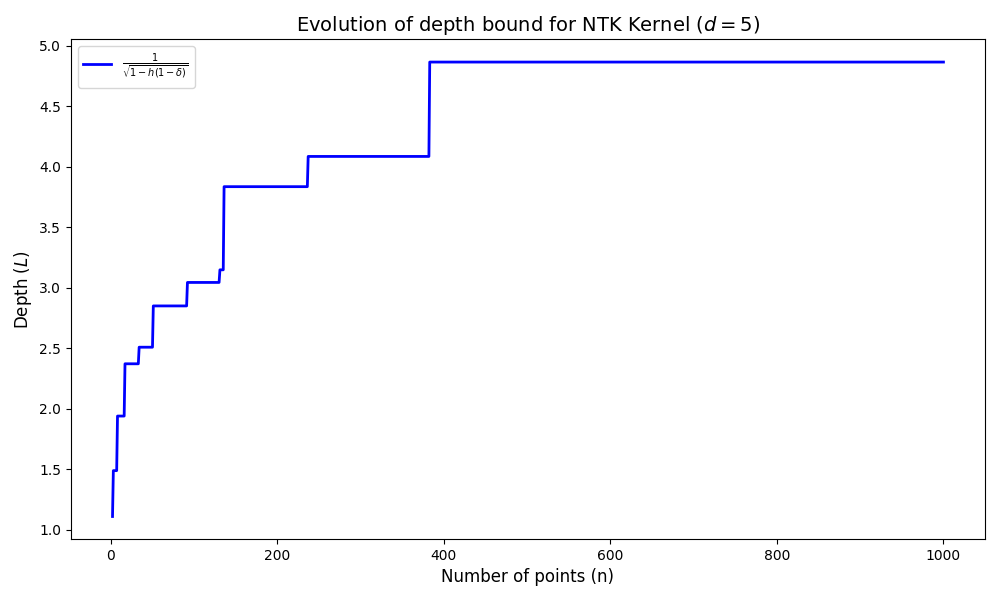}
    \caption{Depth bound, where $L = \max_{i,j \in [n] \mid i \neq j} \frac{1}{\sqrt{1 - h\left( x_i^\top x_j) \right)}}$, as the dataset size $n$ increases. We sampled points on $S^{n_0 -1}$ uniformly with $d=128$ as in Example~\ref{ex:deep_stable}.}
    \label{fig:scale_depth_uniform}
\end{figure}

\section{Conclusion}\label{sec:conclusion}

In this article, we provide a detailed analysis of the evolution of features in the ``lazy regime''. We cast doubt on the use of the term ``lazy'' to describe this learning regime. We show that contrary to previously held beliefs~\citep{Chizatetal2019, yang2021tensor, jacot2018neural}, which only looked at the evolution of single hidden neurons, the full hidden layers evolve during training and do not remain fixed to their initialization as the width increases to infinity. Our results motivate to further study the ``lazy regime'' and to understand what are its generalization properties. While we do not claim that it outperforms the mean-field regime or $\mu P$ parameterization, having access to a closed-form solution can spur the analysis of the properties of deep neural network, such as analyzing the role of depth in generalization.

Motivated by this perspective, we also share insights into the behaviour of the deterministic kernel $\Theta_{\infty}^{(L)}$ as $L \to \infty$ to understand how depth affects a model's predictions. We observe that under the conditions of \citet{jacot2018neural} and with $L \in o(\min_{l=1, \dots, L-1} n_l)$, a fully-connected ReLU neural network can be asymptotically stable for any $x$ given a fixed dataset $X$. A sufficient condition for this behaviour is provided in Theorem~\ref{thm:deep_stable_ntk}, where we control the depth as a function of $n$ and where we perform early stopping. The studied kernel exhibits behaviour consistent with the ordered phase. While our results are for data in $S^{n_0-1}$, we can extend to the general domain as long as there is no colinearity, or by taking the stereographic projection in a space with one additional dimension. Concurrently, we have shown the convergence rates of a non-exhaustive list of kernels as depth $L$ increases. We demonstrate empirically that, with the sublinear convergence to the limiting kernel, we can experimentally control the output $f_t$ in a manner that aligns with experimental settings of most machine learning research. We leave the systematic re-evaluation of the ``lazy regime'' on other domains, leveraging the ideas we presented in this article, as a future research direction. Finally, we provided a list of key properties that were necessary to obtain our results, paving the way for their generalization to other kernels. 

Our framework provides a theoretical foundation for understanding the role of features and depth in the deterministic limiting kernels of overparameterized neural networks. Future research should envisage to characterize the behaviour of kernels that arise in the context of other architectures such as CNNs and architectures with skip connections (e.g. ResNets~\citep{he2016deep}). As mentioned, the setting of \citet{hanin2020finite} is outside the purview of our depth analysis as the stochasticity of the NTK becomes highly relevant when $L \gg \max_{l=1,\dots,L-1} n_l$. Other parameterization regimes should also be studied further, and it would be interesting to explore the properties of learned features in each regime (e.g., task transferability properties and generalization). A systematic and in-depth re-evaluation of feature evolution in the ``lazy regime'' is needed, particularly for computationally intensive tasks: proper feature learning in these cases can motivate passing to the closed-form solution studied in this article to benefit from the complex modelling capability of overparameterized networks at a reduced computational cost.

\acks{This work was funded by the NSERC Grant no. 2024-04051, NSERC Grant no. 2021-04378,
and Canada CIFAR AI Chair. This research was enabled in part by support provided by
Calcul Québec (\url{www.calculquebec.ca}), Compute Canada (\url{www.computecanada.ca}) and
Mila (\url{www.mila.quebec})}.

\bibliography{ref}

\newpage
\appendix

\section{Theoretical assumptions}\label{sec:assumptions}

For the proofs in Section~\ref{subsec:featuresNTK} and~\ref{sec:limiting_kernel}, we make the following list of explicit assumptions:

\begin{enumerate}[align=left]
    \item[\textbf{Activation $\sigma$}:] ReLU activation.
    \item[\textbf{Output dimension}:] The output dimension considered is $n_L = 1$ for a network of depth $L$.
    \item[\textbf{Biases}:] The neural networks do not contain any biases (i.e. $\beta = 0$).
    \item[\textbf{Data}:] The data has a fixed representation on $S^{n_0-1} \subseteq \mathbb{R}^{n_0}$ and no duplicates are in $X$.
    \item[\textbf{Kernel $\Theta_{\infty}^{(L)}$}:] Limiting kernel of a fully-connected ReLU neural network of depth $L$, output dimension $n_L=1$ and without biases, under infinite width (i.e. $\min_{1 \leq l \leq L-1} n_l \to \infty$).
\end{enumerate}

Introducing biases via $\beta > 0$ changes $\Sigma^{(L)}$ and $\Theta_{\infty}^{(L)}$. Therefore, Proposition~\ref{prop:correlated_sigma} does not hold in its exact form. However, we can show that in the limit $L \to \infty$, the kernel $\bar{\Theta}_{\infty}^{(L)}$ converges to an invertible matrix whose off-diagonal entries are all equal. By the same token, we can apply the ideas of Theorem~\ref{thm:deep_stable_ntk}, and we obtain a stable expression for $f_t$. In the general list of properties for $\kappa^{(L)}$ that is found in Section~\ref{sec:experiments}, the value of $\hat{L} = 2$ as the limiting kernel becomes positive definite on the sphere for $L \geq 2$.

\section{Supplementary definitions}\label{app:otherdefinitions}

\begin{deff}[Inverse stereographic projection embedding]\label{def:inverse_stereographic_projection}
    Let $m \in \mathbb{N}$, $q \in S^{m}$, and $\pi_{\text{proj}}: S^{m} \to \mathbb{R}^{m+1}$ be the stereographic projection from $S^{m}$ to $\mathbb{R}^{m+1}$ through the point $q$. The point $q$ can be thought as the ``point at infinity'' in $\mathbb{R}^{m+1}$. The \textbf{inverse stereographic projection embedding} $\pi_{\text{inv}}: \mathbb{R}^{m} \to S^m$ is defined as
    \begin{equation*}
        \pi_{\text{inv}} = \pi_{\text{proj}}^{-1} \circ \pi_{\text{emb}},
    \end{equation*}
    where $\pi_{\text{emb}}: \mathbb{R}^{m} \to \mathbb{R}^{m+1}$ is defined by
    \begin{equation*}
        \pi_{\text{emb}}(x) = (x_1, \dots, x_m, 0).
    \end{equation*}
\end{deff}

\begin{deff}[Logarithmic convergence]\label{def:logarithmic_convergence}
    Given a sequence $\{x_n\}_{n=1}^\infty$, the sequence has a logarithmic convergence rate if $\lim_{n \to \infty} x_n = x$ for some $x \in \mathbb{R}$ and 
    \begin{equation*}
        \lim_{n \to \infty} \frac{\left\lvert x_{n+1} - x \right\rvert}{\left\lvert x_n - x \right\rvert} = 1, \quad
        \lim_{n \to \infty} \frac{\left\lvert x_{n+2} - x_{n+1} \right\rvert}{\left\lvert x_{n+1} - x_n \right\rvert} = 1.
    \end{equation*}
\end{deff}

\section{Additional lemmas, propositions, and theorems}\label{sec:additional_lemmas}

\lemmaappendix*
\begin{proof}
    We observe that $\rho^{(L+1)}(x, x') = h(\rho^{(L)})$, where $h$ is the function
    \begin{equation*}
        h(z) = \frac{z \arcsin(z)}{\pi} + \frac{\sqrt{1 - z^2}}{\pi} + \frac{z}{2}.
    \end{equation*}
    This function has derivative $h'(z) < 1$ on any fixed interval $\left[a, b\right] \subsetneq \left[-1, 1\right]$ such that $b \neq 1$ and is continuously differentiable on the same interval. Therefore, if $H_n(z) = (h \circ \dots \circ h)(z)$ denotes the $n^\text{th}$ power composition of $h$, we have $H_L(z) \to \beta$, some unique fixed-point, uniformly on $\left]-1, 1\right[$. This can be proved by observing that the domain of $H_L$ is a compact set and that for any metric $d$, the distance $d\left( H_L(z_1), H_L(z_2) \right) < d\left( z_1, z_2 \right)$ for $z_1 \neq z_2$ (first show that $d\left( H_L(z), z \right)$ is continuous and has a minimum which is $0$). Furthermore, we have that $h(z) \geq z$ on $\left[-1, 1\right]$, with strict inequality on $\left[-1, 1 \right[$. This implies that $\beta = 1$ and the proof is finished.
\end{proof}

The following proposition provides a proof sketch of the convergence of $\eta^{(L)}$ as $L \to \infty$ and other properties that satisfy the requirements for Theorem~\ref{thm:deep_stable_ntk}.

\begin{prop}[Convergence of $\eta^{(L)}$]\label{prop:convergence_eta}
    The values $\eta^{(L)}(x, x')$ converge to a unique $\beta > 0$ for all $x, x' \in S^{n_0  - 1}$ as $L \to \infty$. Moreover, the kernels $\eta^{(L)}$ are positive definite on $S^{n_0 -1}$ and satisfy $\eta^{(L)}(x, x) \geq \eta^{(L)}(x_1, x_2)$.
\end{prop}
\begin{proof}[Proof sketch]
As  $L \to \infty$, all values converge to the same limit since the derivative of $h$ is strictly smaller than $1$ on $\left[-1, 1 \right]$. The kernels $\eta^{(L)}$ also satisfy $\eta^{(L)}(x, x) \geq \kappa^{(L)}(x_1, x_2)$ since $h(z)$ is monotone increasing in $z \in \left[ -1, 1 \right]$. The kernels are all positive definite on $S^{n_0 - 1}$ since the function $h$ is analytic on $\left[-1, 1\right]$ and it has infinitely many even and odd terms in its power series expansion at $0$ (i.e. not and even or odd function) that are strictly positive~\citep{gneiting_positive_definite}.
\end{proof}

The following is the proof of Proposition~\ref{prop:theta_recursive}.
\thetarecursive*
\begin{proof}
    The fact that the values are all contained in $\left[ 0, 1 \right]$ is immediate from the definition of $\bar{\Theta}_{\infty}^{(L)}$ and Proposition~\ref{prop:closedform}. Now,
    \begin{align*}
        \Theta_{\infty}^{(L+1)}(x, x') &= \dot{\Sigma}^{(L+1)}(x,x') \Theta_{\infty}^{(L)}(x, x') + \Sigma^{(L+1)}(x, x') \\
        &= \frac{1}{2} h'\left(\rho^{(L)}(x, x') \right) \Theta_{\infty}^{(L)}(x, x') \\&+ \frac{1}{n_0 2^L} h\left( \rho^{(L)}(x, x') \right),
    \end{align*}
    where the first equality comes from Theorem~\ref{thm:jacot_ntk} and the second equality uses Propositions~\ref{prop:correlated_sigma} and \ref{prop:closedform}. This implies
    \begin{align*}
        \bar{\Theta}_{\infty}^{(L+1)}(x, x') &= \frac{L n_0 2^L}{(L+1) n_0 2^{L-1}} \\
        &\times \frac{n_0 2^{L-1}\Theta_{\infty}^{(L)}(x, x')}{L} \frac{1}{2} h'\left( \rho^{(L)}(x, x') \right) \\
        &+ \frac{1}{L+1} h\left( \rho^{(L)}(x,x') \right) \\
        &= \frac{L}{L+1} h'\left( \rho^{(L)}(x, x') \right) \bar{\Theta}_{\infty}^{(L)}(x, x') \\
        &+ \frac{1}{L+1} h\left( \rho^{(L)}(x, x') \right),
    \end{align*}
    where the equalities come from the definition of $\Theta_{\infty}^{(L+1)}(x, x')$ and the normalization factors of $\bar{\Theta}_{\infty}^{(L+1)}$ and $\bar{\Theta}_{\infty}^{(L)}$.
\end{proof}

\section{Further experiments}\label{sec:appendix_further_experiments}

In this section, we present further experiments relating to the theoretical results presented in Section~\ref{sec:limiting_kernel}. We report the mean squared error, accuracy and F1 score on MNIST data in Table~\ref{tab:kernel_generalization_performance_mnist}. For a small dataset of size $500$ and a depth of $30$, we observe good performance. In Figure~\ref{fig:convergence_rates_mnist}, we can see that the limiting solution $\bar{\Theta}_{\infty}^{(L)} \left( X X^\top \right )$ is well approximated at such depths. Therefore, we can model the output of a very wide and deep network using minimal resources. 

\begin{table}[ht]
    \centering
    \begin{tabular}{cc}
        \toprule
        \textbf{Metric} & \\
        \midrule
        Average error (stdev) & $6.435 \, (0.075)$ \\
        Accuracy & $0.874$\\
        F1 score & $0.870$ \\
        \bottomrule
    \end{tabular}
    \caption{Various metrics on kernel regression model for MNIST. The kernel used is $\bar{\Theta}_{\infty}^{(L)} \left( X X^\top \right)$. The size of the dataset was set to $\left\lvert X \right\rvert = 500$ and the depth to $L=30$.}
    \label{tab:kernel_generalization_performance_mnist}
\end{table}

\end{document}